\documentclass[11pt]{article} 
\usepackage{times}
\usepackage[margin=1in]{geometry}
\usepackage{graphics,graphicx,authblk} 
\usepackage{url,amsmath,amssymb,amsthm,flushend,multicol,fullpage,adjustbox}
\usepackage{natbib}
\usepackage{pifont,color}
\usepackage{placeins}
\usepackage{algorithm}
\usepackage{algorithmic,bm,multirow,array}

\usepackage{enumitem,scalerel}
\usepackage{xcolor}
\usepackage[colorlinks = true,
            linkcolor = red,
            urlcolor  = blue,
            citecolor = blue,
            anchorcolor = blue]{hyperref}

\usepackage{subfigure}
\usepackage{xspace,hhline}

\newcommand{\infbound}{R^2}

\renewcommand{\vec}[1]{\mathbf{#1}}
\newcommand{\mat}[1]{\mathbf{#1}}

\newcommand{\R}{\mathbb{R}}

\renewcommand{\l}{l}
\newcommand{\w}{\vec{w}}

\renewcommand{\c}{c}

\newcommand{\xa}{\widehat{a}}
\newcommand{\ya}{\widehat{b}}
\newcommand{\za}{\widehat{c}}
\newcommand{\va}{\widehat{d}}
\newcommand{\alpp}{\alpha_1}
\newcommand{\betp}{\beta_1}
\newcommand{\gamp}{\gamma_1}
\newcommand{\delp}{\delta_1}

\newcommand{\iprod}[2]{\left\langle #1, #2 \right\rangle}

\newcommand{\E}[1]{\mathbb{E}\left[#1\right]}

\newcommand{\D}{\mathcal{D}}

\newcommand{\Cov}{\mat{H}}

\newcommand{\Hhat}{\widehat{\Cov}}

\newcommand{\eqdef}{\stackrel{\textrm{def}}{=}}
\newcommand{\cnS}{\widetilde{\kappa}}
\newcommand{\cnH}{{\kappa}}

\newcommand{\abs}[1]{\left|{#1}\right|}

\renewcommand{\H}{\mat{H}}
\newcommand{\phiv}{\boldsymbol{\Phi}}

\newcommand{\tensor}[1]{\mathcal{#1}}
\newcommand{\BT}{\tensor{B}}
\newcommand{\DT}{\tensor{D}}
\newcommand{\RT}{\tensor{R}}

\newcommand{\eyeT}{\tensor{I}}

\newcommand{\cone}{c_1}

\newcommand{\cthree}{c_3}

\newcommand{\thetav}{\boldsymbol{\theta}}

\newcommand{\Ah}{\widehat{\mat{A}}}

\DeclareMathOperator*{\poly}{poly}

\newcommand{\init}{k}
\renewcommand{\th}{\hat{t}}

\newcommand{\hb}{HB\xspace}
\newcommand{\nag}{NAG\xspace}
\newcommand{\sgd}{SGD\xspace}
\newcommand{\asgd}{ASGD\xspace}

\newcommand{\sfo}{stochastic first order oracle\xspace}
\newcommand{\sfoa}{SFO\xspace}
\newcommand{\cifar}{cifar-10\xspace}
\newcommand{\mnist}{mnist\xspace}

\newcommand{\gd}{GD\xspace}

\title{On the insufficiency of existing momentum schemes for Stochastic Optimization\footnote{Appeared as an oral presentation at International Conference on Learning Representations (ICLR), 2018, Vancouver, Canada.}}

\author[1]{Rahul Kidambi}
\author[2]{Praneeth Netrapalli}
\author[2]{Prateek Jain}
\author[1]{Sham M. Kakade}
\affil[1]{University of Washington, Seattle, WA, USA, \url{rkidambi@uw.edu},\ \url{sham@cs.washington.edu}}
\affil[2]{Microsoft Research, Bangalore, India,    \url{{praneeth,prajain}@microsoft.com}}

\date{}
\allowdisplaybreaks
\newtheorem{theorem}{Theorem}
\newtheorem{lemma}[theorem]{Lemma}
\newtheorem{corollary}[theorem]{Corollary}
\newtheorem{remark}[theorem]{Remark}

\newtheorem{proposition}[theorem]{Proposition}
\theoremstyle{definition}

\begin{document}

\maketitle

\begin{abstract}
Momentum based stochastic gradient methods such as heavy ball (\hb)
and Nesterov's accelerated gradient descent (\nag) method are widely
used in practice for training deep networks and other
supervised learning models, as they often provide significant
improvements over stochastic gradient descent (\sgd).
Rigorously speaking, ``fast gradient'' methods have provable improvements over
gradient descent only for the deterministic case, where the gradients
are exact. In the stochastic case, the popular explanations for their wide
applicability is that when these fast gradient methods are applied in
the stochastic case, they partially mimic their {\em exact} gradient
counterparts, resulting in some practical gain. This work provides a
counterpoint to this belief by {\em proving} that there exist simple
problem instances where these methods {\em cannot} outperform \sgd
despite the best setting of its parameters. These negative problem
instances are, in an informal sense, generic; they do not look like
carefully constructed pathological instances. These results suggest
(along with empirical evidence) that \hb or \nag's practical
performance gains are a by-product of mini-batching.

Furthermore, this work provides a viable (and provable) alternative,
which, on the same set of problem instances, significantly improves
over \hb, \nag, and \sgd's performance. This algorithm, referred to as 
Accelerated Stochastic Gradient Descent (\asgd), is a 
simple to implement stochastic algorithm, based on a relatively less
popular variant of Nesterov's Acceleration. Extensive empirical results in this paper show
that \asgd has performance
gains over \hb, \nag, and \sgd. The code implementing the \asgd Algorithm can be found \href{https://github.com/rahulkidambi/AccSGD}{here}\footnote{link to the \asgd code: \href{https://github.com/rahulkidambi/AccSGD}{https://github.com/rahulkidambi/AccSGD}}.

\end{abstract}

\section{Introduction}\label{sec:intro}
First order optimization methods, which access a function to be optimized through its gradient or an unbiased approximation of its gradient, are the workhorses for modern large scale optimization problems, which include training the current state-of-the-art deep neural networks. Gradient descent~\citep{cauchy1847} is the simplest first order method that is used heavily in practice. However, it is known that for the class of smooth convex functions  as well as some simple non-smooth problems~\citep{Nesterov12b}), gradient descent is suboptimal~\citep{Nesterov04} and there exists a class of algorithms called fast gradient/momentum based methods which achieve optimal convergence guarantees. The heavy ball method~\citep{polyak1964some} and Nesterov's accelerated gradient descent~\citep{nesterov1983method} are two of the most popular methods in this category. 

On the other hand, training deep neural networks on large scale datasets have been possible through the use of Stochastic Gradient Descent (\sgd)~\citep{RobbinsM51}, which samples a random subset of training data to compute gradient estimates that are then used to optimize the objective function. The advantages of \sgd for large scale optimization and the related issues of tradeoffs between computational and statistical efficiency was highlighted in~\citet{BottouB07}.

The above mentioned theoretical advantages of fast gradient methods~\citep{polyak1964some,nesterov1983method} (albeit for smooth convex problems) coupled with cheap to compute stochastic gradient estimates led to the influential work of~\citet{sutskever2013importance}, which demonstrated the empirical advantages possessed by \sgd when augmented with the momentum machinery. This work has led to widespread adoption of momentum methods for training deep neural nets; so much so that, in the context of neural network training, gradient descent often refers to momentum methods.


But, there is a subtle difference between classical momentum methods and their implementation in practice -- classical momentum methods work in the exact first order oracle model~\citep{Nesterov04}, i.e., they employ \emph{exact gradients} (computed on the full training dataset), while in practice~\citep{sutskever2013importance}, they are implemented with \emph{stochastic gradients} (estimated from a randomly sampled mini-batch of training data). This leads to a natural question:

 {\em ``Are momentum methods optimal even in the {\bf stochastic first order oracle} (SFO) model, where we access stochastic gradients computed on a small constant sized minibatches (or a batchsize	 of $1$?)''}

Even disregarding the question of optimality of momentum methods in the \sfoa model, it is not even known if momentum methods (say,~\citet{polyak1964some,nesterov1983method}) provide \emph{any provable} improvement over \sgd in this model. While these are open questions, a recent effort of~\cite{jain2017accelerating} showed that improving upon~\sgd~(in the~\sfo) is rather subtle as there exists problem instances in \sfoa model where it is not possible to improve upon \sgd, even information theoretically. \citet{jain2017accelerating} studied a variant of Nesterov's accelerated gradient updates~\citep{Nesterov12} for stochastic linear regression and show that their method improves upon~\sgd~wherever it is information theoretically admissible. Through out this paper, we refer to the algorithm of~\citet{jain2017accelerating} as Accelerated Stochastic Gradient Method (\asgd) while we refer to a stochastic version of the most widespread form of Nesterov's method~\citep{nesterov1983method} as \nag; \hb denotes  a stochastic version of the heavy ball method~\citep{polyak1964some}. Critically, while \citet{jain2017accelerating} shows that \asgd improves on \sgd in any information-theoretically admissible regime, it is still not known whether \hb and \nag can achieve a similar performance gain. 

A key contribution of this work is to show that \hb\ {\em does not} provide similar performance gains over \sgd even when it is informationally-theoretically admissible. That is, we provide a problem instance where it is indeed possible to improve upon \sgd (and \asgd achieves this improvement), but \hb\ {\em cannot} achieve any improvement over \sgd. We validate this claim empirically as well. In fact, we provide empirical evidence to the claim that \nag\ also do not achieve any improvement over \sgd\ for several problems where \asgd can still achieve better rates of convergence.

This raises a question about why \hb and \nag provide better performance than \sgd in practice \citep{sutskever2013importance}, especially for training deep networks. Our conclusion (that is well supported by our theoretical result) is that \hb and \nag's improved performance is attributed to mini-batching and hence, these methods will often struggle to improve over \sgd with small constant batch sizes. 
This is in stark contrast to methods like \asgd, which is designed to improve over \sgd across both small or large mini-batch sizes. In fact, based on our experiments, we observe that on the task of training deep residual networks~\citep{kaiminghe2016identity} on the \cifar dataset, we note that \asgd offers noticeable improvements by achieving $5-7\%$ better test error over \hb and \nag even with commonly used batch sizes like $128$ during the initial stages of the optimization. 

\subsection{Contributions}
The contributions of this paper are as follows.
\begin{enumerate}
	\item  In Section~\ref{sec:hb}, we prove that~\hb~is \emph{not optimal} in the~\sfoa~model. In particular, there exist linear regression problems for which the performance of~\hb~(with \emph{any} step size and momentum) is either the same or worse than that of~\sgd~while~\asgd~ improves upon both of them.
	\item Experiments on several linear regression problems suggest that the suboptimality of \hb in the~\sfoa model is not restricted to special cases -- it is rather widespread. Empirically, the same holds true for~\nag~as well (Section~\ref{sec:exp}). 
	\item The above observations suggest that the only reason for the superiority of momentum methods in practice is mini-batching, which reduces the variance in stochastic gradients and moves the \sfoa closer to the exact first order oracle. This conclusion is supported by empirical evidence through training deep residual networks on~\cifar, with a batch size of $8$ (see Section~\ref{sec:cifar}).
	\item We present an intuitive and easier to tune version of \asgd (see Section~\ref{sec:algo}) and show that \asgd can provide significantly faster convergence to a reasonable accuracy than ~\sgd,~\hb,~\nag, while still providing favorable or comparable asymptotic accuracy as these methods, particularly on several deep learning problems. 
\end{enumerate}
Hence, the take-home message of this paper is: {\em \hb~and~\nag~are not optimal in the~\sfoa model. The only reason for the superiority of momentum methods in practice is mini-batching.~\asgd~provides a distinct advantage in training deep networks over~\sgd,~\hb~and~\nag.}

\section{Notation}\label{sec:notation}
We denote matrices by bold-face capital letters and vectors by lower-case letters. $f(w)=1/n\sum_i f_i(w)$ denotes the function to optimize w.r.t. model parameters $w$. $\nabla f(w)$ denotes exact  gradient of $f$ at $w$ while $\widehat{\nabla} f_t(w)$ denotes a stochastic gradient of $f$. That is, $\widehat{\nabla} f_t(w_t)=\nabla f_{i_t}(w)$ where $i_t$ is sampled uniformly at random from $[1,\dots, n]$. For linear regression, $f_i(w)=0.5\cdot(b_i-\iprod{w}{a_i})^2$ where $b_i\in\Re$ is the target and $a_i\in \Re^d$ is the covariate, and $\widehat{\nabla} f_t(w_t)=-(b_t-\iprod{w_t}{a_t})a_t$. In this case, $\H=\E{aa^{\top}}$ denotes the Hessian of $f$ and $\kappa=\tfrac{\lambda_1(\H)}{\lambda_d(\H)}$ denotes it's condition number. 

Algorithm~\ref{alg:hb} provides a pseudo-code of \hb method~\citep{polyak1964some}. $w_t-w_{t-1}$ is the momentum term and $\alpha$ denotes the momentum parameter. Next iterate $w_{t+1}$ is obtained by a linear combination of the \sgd update and the momentum term. Algorithm~\ref{alg:nag} provides pseudo-code of a stochastic version of the most commonly used form of Nesterov's accelerated gradient descent~\citep{nesterov1983method}.

\begin{figure}
	\begin{minipage}{.5\textwidth}
		\begin{algorithm}[H]
			\caption{\hb: Heavy ball with a~\sfoa}
			\label{alg:hb}
			\begin{algorithmic}[1]
				\REQUIRE Initial  $w_0$, stepsize $\delta$, momentum $\alpha$
				\STATE	${w}_{-1} \leftarrow w_0$; $t\leftarrow 0$\hfill
				/*Set $w_{-1}$ to $w_0$*/
				\WHILE{$w_t$ not converged}
				\STATE	$w_{t+1}\leftarrow w_t - \delta \cdot \widehat{\nabla} f_t(w_t) + \alpha \cdot \left({w}_{t}-w_{t-1}\right)$\hfill /*Sum of stochastic gradient step and momentum*/
				\STATE $t\leftarrow t+1$
				\ENDWHILE
				\ENSURE $w_t$ \hfill /*Return the last iterate*/
			\end{algorithmic}
		\end{algorithm}
	\end{minipage}\hspace*{5pt}
	\begin{minipage}{.5\textwidth}
		\begin{algorithm}[H]
			\caption{\nag: Nesterov's AGD with a~\sfoa}
			\label{alg:nag}
			\begin{algorithmic}[1]
				\REQUIRE Initial  $w_0$, stepsize $\delta$, momentum $\alpha$
				\STATE	${v}_0 \leftarrow w_0$; $t\leftarrow 0$ \hfill
				/*Set $v_0$ to $w_0$*/
				\WHILE{$w_t$ not converged}
				\STATE	$v_{t+1}\leftarrow w_t - \delta \cdot \widehat{\nabla} f_t(w_t) $ /*SGD step*/
				\STATE $w_{t+1}=(1+\alpha)v_{t+1}-\alpha v_t $\hfill /*Sum of SGD step and previous iterate*/
				\STATE $t\leftarrow t+1$
				\ENDWHILE
				\ENSURE $w_t$ \hfill /*Return the last iterate*/
			\end{algorithmic}
		\end{algorithm}
	\end{minipage}
\end{figure}

\section{Suboptimality of Heavy Ball Method}
\label{sec:hb}
In this section, we show that there exists linear regression problems where the performance of~\hb~ (Algorithm~\ref{alg:hb}) is no better than that of~\sgd, while~\asgd~significantly improves upon \sgd's performance.
~Let us now describe the problem instance.

Fix $w^*\in \R^2$ and let $(a,b)\sim\D$ be a sample from the distribution such that: 
\begin{align*}
	a = \begin{cases}
	\sigma_1 \cdot z \cdot e_1 \mbox{ w.p. } 0.5 \\
	\sigma_2 \cdot z \cdot e_2 \mbox{ w.p. } 0.5,
	\end{cases} \qquad \mbox{ and } \qquad b = \iprod{w^*}{a},
\end{align*}
where $e_1,e_2\in \R^2$ are canonical basis vectors, $\sigma_1 > \sigma_2 > 0$. Let $z$ be a random variable such that $\E{z^2}=2$ and $\E{z^4} = 2\c \geq 4$. Hence, we have: $\E{(a^{(i)})^2}=\sigma_i^2, \E{(a^{(i)})^4}=\c\sigma_i^4,$ for $i=1,2$. Now, our goal is to minimize: $$f(w) \eqdef 0.5\cdot\E{\left(\iprod{w^*}{a}-b\right)^2}, \text{  Hessian  } \H \eqdef \E{aa^\top} = \begin{bmatrix}
\sigma_1^2 & 0 \\ 0 & \sigma_2^2
\end{bmatrix}.$$ 
Let $\kappa$ and $\tilde{\kappa}$ denote the {\em computational} and {\em statistical} condition numbers -- see~\cite{jain2017accelerating} for definitions. For the problem above, we have $\kappa = \frac{c \sigma_1^2}{\sigma_2^2}$ and $\tilde{\kappa} = c$. Then we obtain following convergence rates for \sgd and \asgd when applied to the above given problem instance: 
\begin{corollary}[of Theorem~$1$ of~\cite{jain2016parallelizing}]\label{cor:sgd}
	Let $w_t^{\sgd}$ be the $t^\textrm{th}$ iterate of~\sgd~on the above problem with starting point $w_0$ and stepsize $\frac{1}{c\sigma_1^2}$. The error of $w_t^{\sgd}$ can be bounded as,
	\begin{align*}
		\E{f\left({w_t^{\sgd}}\right)}-f\left(w_*\right) \leq \exp\left(\frac{-t}{\kappa}\right) 		\big(f\left({w_0}\right)-f\left(w_*\right)\big).
	\end{align*}
\end{corollary}
On the other hand,~\asgd~achieves the following superior rate.
\begin{corollary}[of Theorem~$1$ of~\cite{jain2017accelerating}]\label{cor:asgd}
	Let $w_t^{\asgd}$ be the $t^\textrm{th}$ iterate of~\asgd~on the above problem with starting point $w_0$ and appropriate parameters. The error of $w_t^{\asgd}$ can be bounded as,
	\begin{align*}
	\E{f\left({w_t^{\asgd}}\right)}-f\left(w_*\right) \leq \poly(\kappa) \exp\left(\frac{-t}{\sqrt{\kappa\tilde{\kappa}}}\right) \big(f\left({w_0}\right)-f\left(w_*\right)\big).
	\end{align*}
\end{corollary}
Note that for a given problem/input distribution $\tilde{\kappa}=c$ is a constant while $\kappa=\frac{c \sigma_1^2}{\sigma_2^2}$ can be arbitrarily large. Note that $\kappa>\tilde{\kappa}=c$. Hence, \asgd improves upon rate of \sgd by a factor of $\sqrt{\kappa}$. The following proposition, which is the main result of this section, establishes that ~\hb~(Algorithm~\ref{alg:hb}) cannot provide a similar improvement over \sgd as what \asgd offers. In fact, we show no matter the choice of parameters of \hb, its performance does not improve over \sgd by more than a constant.
\begin{proposition}\label{prop:hb}
	Let $w_t^{\hb}$ be the $t^\textrm{th}$ iterate of~\hb~(Algorithm~\ref{alg:hb}) on the above problem with starting point $w_0$. For \emph{any} choice of stepsize $\delta$ and momentum $\alpha\in[0,1]$, $\exists T$ large enough such that $\forall t \geq T$, we have,
	\begin{align*}
	\E{f\left({w_t^{\hb}}\right)}-f\left(w_*\right) \geq C(\kappa,\delta,\alpha) \cdot \exp\left(\frac{-500t}{\kappa}\right) 		\big(f\left({w_0}\right)-f\left(w_*\right)\big),
	\end{align*}
	where $C(\kappa,\delta,\alpha)$ depends on $\kappa, \delta$ and $\alpha$ (but not on $t$).
\end{proposition}
Thus, to obtain $\widehat{w}$ s.t. $\|\widehat{w}-w^*\|\leq \epsilon$, \hb\ requires $\Omega(\kappa \log \frac{1}{\epsilon})$ samples and iterations. On the other hand, \asgd can obtain $\epsilon$-approximation to $w^*$ in $\mathcal{O}(\sqrt{\kappa}\log{\kappa} \log \frac{1}{\epsilon})$ iterations. We note that the gains offered by \asgd are meaningful when $\kappa>\mathcal{O}(c)$~\citep{jain2017accelerating}; otherwise, all the algorithms including \sgd achieve nearly the same rates (upto constant factors). While we do not prove it theoretically, we observe empirically that for the same problem instance, \nag also obtains nearly same rate as \hb and \sgd. We conjecture that a lower bound for \nag can be established using a similar proof technique as that of \hb (i.e. Proposition~\ref{prop:hb}). We also believe that the constant in the lower bound described in proposition~\ref{prop:hb} can be improved to some small number ($\leq 5$).\vspace*{-2mm}

\section{Algorithm}\label{sec:algo}
\begin{algorithm}[t!]
	\caption{Accelerated stochastic gradient descent --~\asgd} 
	\label{alg:asgd}
	\begin{algorithmic}[1]
	\renewcommand{\algorithmicrequire}{\textbf{Input: }}
	\renewcommand{\algorithmicensure}{\textbf{Output: }}
	\REQUIRE Initial $w_0$, short step $\delta$, long step parameter $\kappa \geq 1$, statistical advantage parameter $\xi \leq \sqrt{\kappa}$
	\STATE	$\bar{w}_0 \leftarrow w_0$;  $t \leftarrow 0$\hfill
		/*Set running average to $w_0$*/
	\STATE $\alpha \leftarrow 1 - \frac{0.7^2\cdot\xi}{\kappa}$ \hfill /*Set momentum value*/
	\WHILE{$w_t$ not converged}
	\STATE $\bar{w}_{t+1} \leftarrow \alpha\cdot \bar{w}_{t} + (1-\alpha) \cdot \left(w_{t}-\frac{\kappa\cdot\delta}{0.7} \cdot \widehat{\nabla} f_t(w_t)\right)$ \hfill /*Update the running average as a weighted average of previous running average and a long step gradient */
	\STATE	$w_{t+1}\leftarrow \frac{0.7}{0.7+(1-\alpha)} \cdot \left(w_t - \delta \cdot \widehat{\nabla} f_t(w_t)\right) + \frac{1-\alpha}{0.7+(1-\alpha)} \cdot \bar{w}_{t+1}$\hfill /*Update the iterate as weighted average of current running average and short step gradient*/
	\STATE $t\leftarrow t+1$
	\ENDWHILE
	\ENSURE $w_t$ \hfill /*Return the last iterate*/
	\end{algorithmic}
\end{algorithm}
We will now present and explain an intuitive version of~\asgd (pseudo code in Algorithm~\ref{alg:asgd}). The algorithm takes three inputs: short step $\delta$, long step parameter $\kappa$ and statistical advantage parameter $\xi$. The short step $\delta$ is precisely the same as the step size in~\sgd,~\hb~or~\nag. For convex problems, this scales inversely with the smoothness of the function. The long step parameter $\kappa$ is intended to give an estimate of the ratio of the largest and smallest curvatures of the function; for convex functions, this is just the condition number. The statistical advantage parameter $\xi$ captures trade off between statistical and computational condition numbers -- in the deterministic case, $\xi = \sqrt{\kappa}$ and~\asgd~is equivalent to~\nag, while in the high stochasticity regime, $\xi$ is much smaller. The algorithm maintains two iterates: descent iterate $w_t$ and a running average $\bar{w}_t$. The running average is a weighted average of the previous average and a long gradient step from the descent iterate, while the descent iterate is updated as a convex combination of short gradient step from the descent iterate and the running average. The idea is that since the algorithm takes a long step as well as short step and an appropriate average of both of them, it can make progress on different directions at a similar pace. Appendix~\ref{app:equiv} shows the equivalence between Algorithm~\ref{alg:asgd} and~\asgd~as proposed in~\cite{jain2017accelerating}. Note that the constant $0.7$ appearing in Algorithm~\ref{alg:asgd} has no special significance.~\cite{jain2017accelerating} require it to be smaller than $\sqrt{1/6}$ but any constant smaller than $1$ seems to work in practice.\vspace*{-2mm}

\section{Experiments}\label{sec:exp}
We now present our experimental results exploring performance of~\sgd,~\hb,~\nag~and~\asgd. Our experiments are geared towards answering the following questions:
\begin{itemize}
	\item Even for linear regression, is the suboptimality of~\hb~restricted to specific distributions in Section~\ref{sec:hb} or does it hold for more general distributions as well? Is the same true of~\nag?
	\item What is the reason for the superiority of~\hb~and~\nag~in practice? Is it because momentum methods have better performance that \sgd for stochastic gradients or due to mini-batching? Does this superiority hold even for small minibatches?
	\item How does the performance of~\asgd~compare to that of~\sgd,~\hb~and~\nag, when training deep networks?
\end{itemize}
Section~\ref{sec:synthetic} and parts of Section~\ref{sec:mnist} address the first two questions. Section~\ref{sec:mnist} and \ref{sec:cifar} address Question 2 partially and the last question. We use Matlab to conduct experiments presented in Section~\ref{sec:synthetic} and use PyTorch \citep{pytorch} for our deep networks related experiments. Pytorch code implementing the \asgd algorithm can be found at \href{https://github.com/rahulkidambi/AccSGD}{https://github.com/rahulkidambi/AccSGD}.

\subsection{Linear Regression}\label{sec:synthetic}
In this section, we will present results on performance of the four optimization methods (\sgd, \hb, \nag, and \asgd) for linear regression problems. We consider two different class of linear regression problems, both in two dimensions. For a given condition number $\kappa$, we consider the following two distributions: 

\textbf{Discrete}: $a=e_1$ w.p. $0.5$ and $a=\frac{2}{\kappa} \cdot e_2$ with $0.5$; $e_i$ is the $i^{th}$ standard basis vector. 

\textbf{Gaussian	}: $a\in \R^2$ is distributed as a Gaussian random vector with covariance matrix $\begin{bmatrix}
1 & 0 \\ 0 & \frac{1}{\kappa}
\end{bmatrix}$.

We fix a randomly generated $w^* \in \R^2$ and for both the distributions above, we let $b = \iprod{w^*}{a}$. We vary $\kappa$ from $\{2^4, 2^5,...,2^{12}\}$ and for each $\kappa$ in this set, we run 100 independent runs of all four methods, each for a total of $t=5\kappa$ iterations. We define that the algorithm converges if there is no error in the second half (i.e. after $2.5\kappa$ updates) that exceeds the starting error - this is reasonable since we expect {\em geometric convergence} of the initial error. 

Unlike \asgd and \sgd, we do not know optimal learning rate and momentum parameters for \nag and \hb in the stochastic gradient model. So, we perform a {\em grid search} over the values of the learning rate and momentum parameters. In particular, we lay a $10\times10$ grid in $[0,1]\times[0,1]$ for learning rate and momentum and run \nag and \hb. Then, for each grid point, we consider the subset of $100$ trials that converged and computed the final error using these. Finally, the parameters that yield the minimal error are chosen for \nag and \hb, and these numbers are reported. We measure convergence performance of a method using:
\begin{align}\label{eq:rate}
\text{rate} = \frac{\log(f(w_0))-\log(f(w_t))}{t},
\end{align}

\begin{figure}[ht!]
	\centering
	\subfigure{\includegraphics[width=0.45\columnwidth]{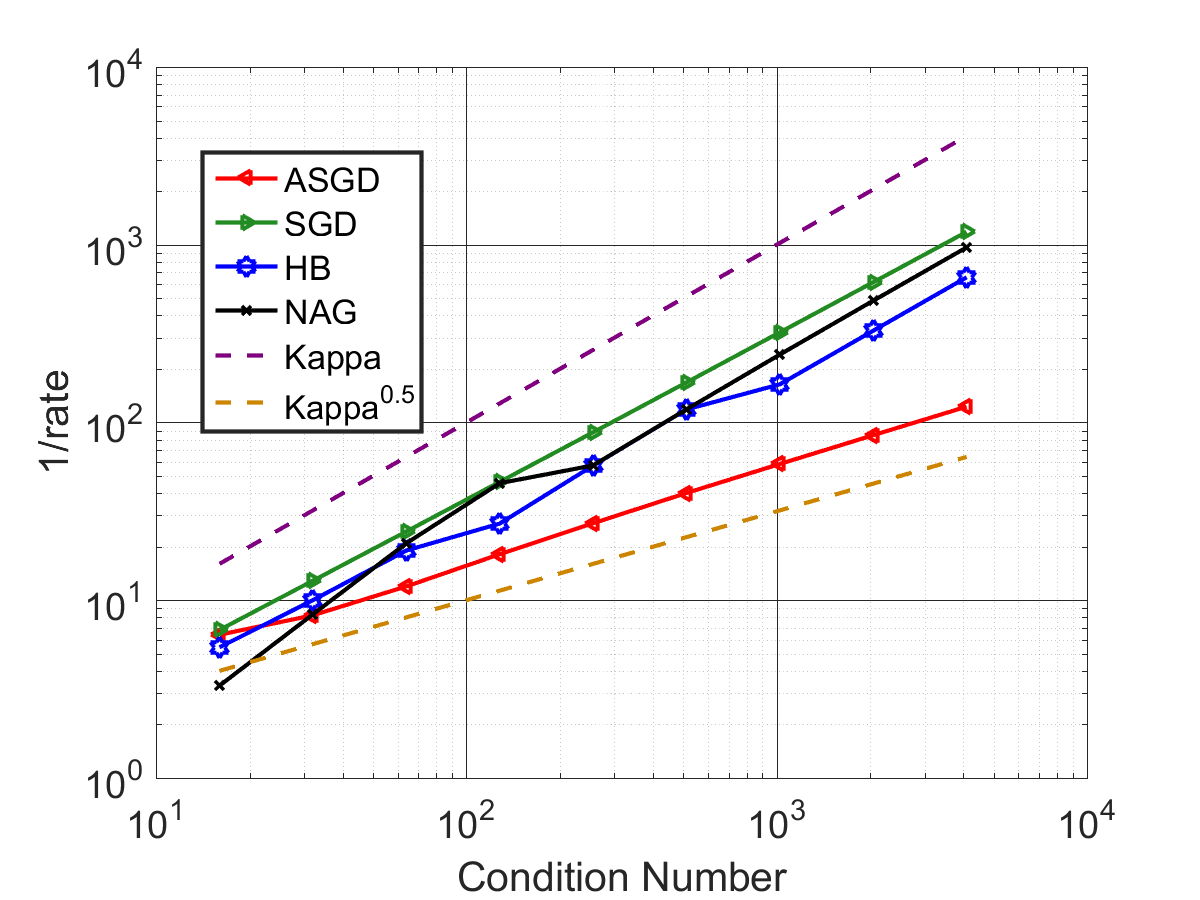}}
	\subfigure{\includegraphics[width=0.45\columnwidth]{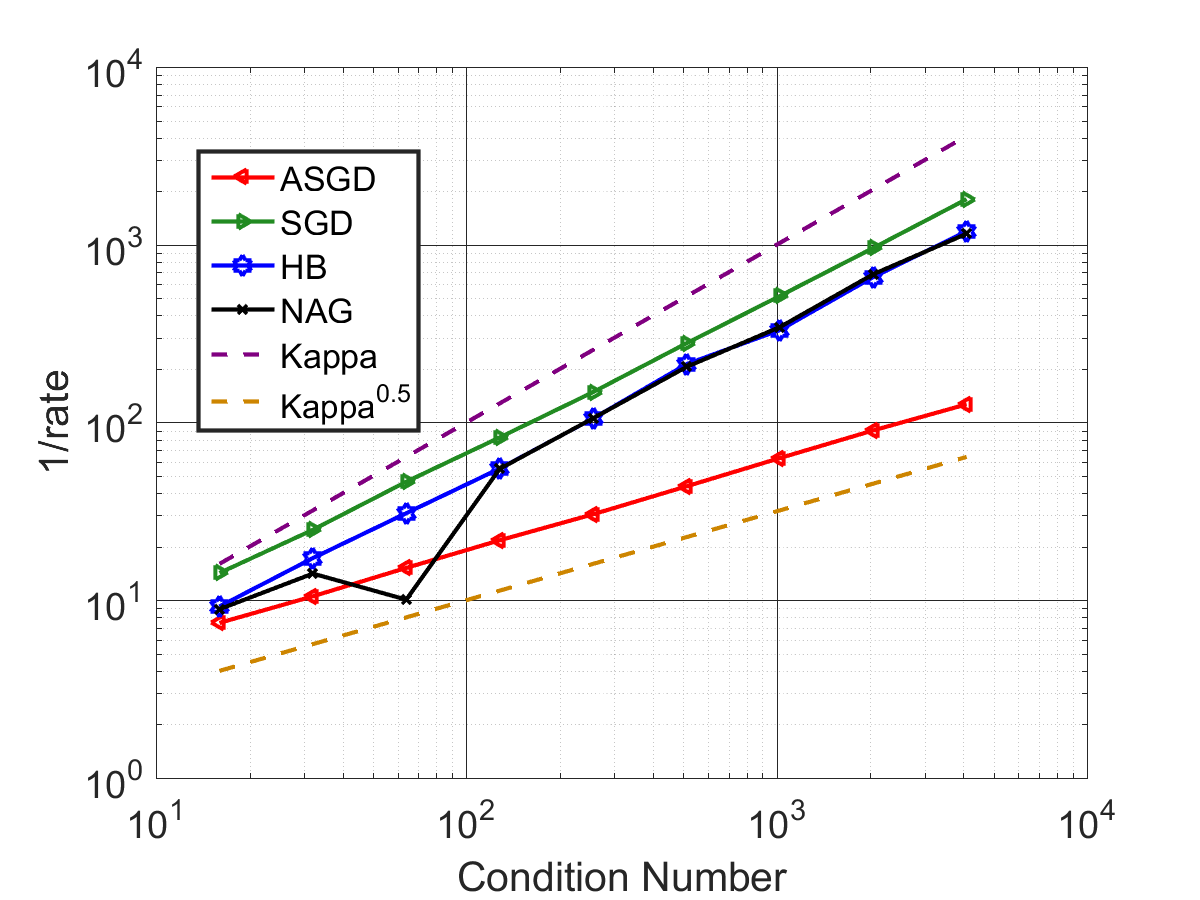}}
	\caption{Plot of 1/rate (refer equation~\eqref{eq:rate}) vs condition number ($\kappa$) for various methods for the linear regression problem. Discrete distribution in the left, Gaussian to the right.}
	\label{fig:synthetic}
\end{figure}

\begin{table}[t]
	\centering
	\begin{tabular}{|c | c | c|}
		\hline
		Algorithm & Slope -- discrete & Slope -- Gaussian \\\hhline{|=|=|=|}
		\sgd & 0.9302 & 0.8745 \\ 
		\hb & 0.8522 & 0.8769 \\
		\nag & 0.98 & 0.9494 \\
		\asgd & 0.5480 & 0.5127 \\ [1ex] 
		\hline
	\end{tabular}
	\caption{Slopes (i.e. $\gamma$) obtained by fitting a line to the curves in Figure~\ref{fig:synthetic}. A value of $\gamma$ indicates that the error decays at a rate of $\exp\left(\frac{-t}{\kappa^\gamma}\right)$. A smaller value of $\gamma$ indicates a faster rate of error decay.}
	\label{table:synthetic}\vspace*{-2mm}
\end{table}

We compute the rate ~\eqref{eq:rate} for all the algorithms with varying condition number $\kappa$. Given a rate vs $\kappa$ plot for a method, we compute it's {\em slope} (denoted as $\gamma$) using linear regression. 
Table~\ref{table:synthetic} presents the estimated slopes (i.e. $\gamma$) for various methods for both the discrete and the Gaussian case. The slope values clearly show that the rate of~\sgd,~\hb~and~\nag~have a nearly linear dependence on $\kappa$ while that of~\asgd~seems to scale linearly with $\sqrt{\kappa}$.
\vspace*{-2mm}


\subsection{Deep Autoencoders for MNIST}\label{sec:mnist}
In this section, we present experimental results on training deep autoencoders for the~\mnist dataset, and we follow the setup of~\citet{hinton2006reducing}. This problem is a standard benchmark for evaluating optimization algorithms e.g.,~\citet{martens2010,sutskever2013importance,martens2015kronecker,Reddi2017escaping}. The network architecture follows previous work~\citep{hinton2006reducing} and is represented as $784-1000-500-250-30-250-500-1000-784$ with the first and last $784$ nodes representing the input and output respectively. All hidden/output nodes employ sigmoid activations except for the layer with $30$ nodes which employs linear activations and we use MSE loss. We use the initialization scheme of~\citet{martens2010}, also employed in~\citet{sutskever2013importance,martens2015kronecker}. We perform training with two minibatch sizes $-1$ and $8$. The runs with minibatch size of $1$ were run for $30$ epochs while the runs with minibatch size of $8$ were run for $50$ epochs. For each of~\sgd,~\hb,~\nag~and~\asgd, a grid search over learning rate, momentum and long step parameter (whichever is applicable) was done and best parameters were chosen based on achieving the smallest training error in the same protocol followed by~\citet{sutskever2013importance}. The grid was extended whenever the best parameter fell at the edge of a grid. For the parameters chosen by grid search, we perform $10$ runs with different seeds and averaged the results. The results are presented in Figures~\ref{fig:mnist-train-test-bs8} and~\ref{fig:mnist-train-test-bs1}. Note that the final loss values reported are suboptimal compared to the published literature e.g.,~\citet{sutskever2013importance}; while~\citet{sutskever2013importance} report results after $750000$ updates with a large batch size of $200$ (which implies a total of $750000\times200=150$M gradient evaluations), whereas, ours are after $1.8$M updates of ~\sgd with a batch size $1$ (which is just $1.8$M gradient evaluations). 
\begin{figure}[t]
	\centering
	\subfigure{\includegraphics[width=0.45\columnwidth]{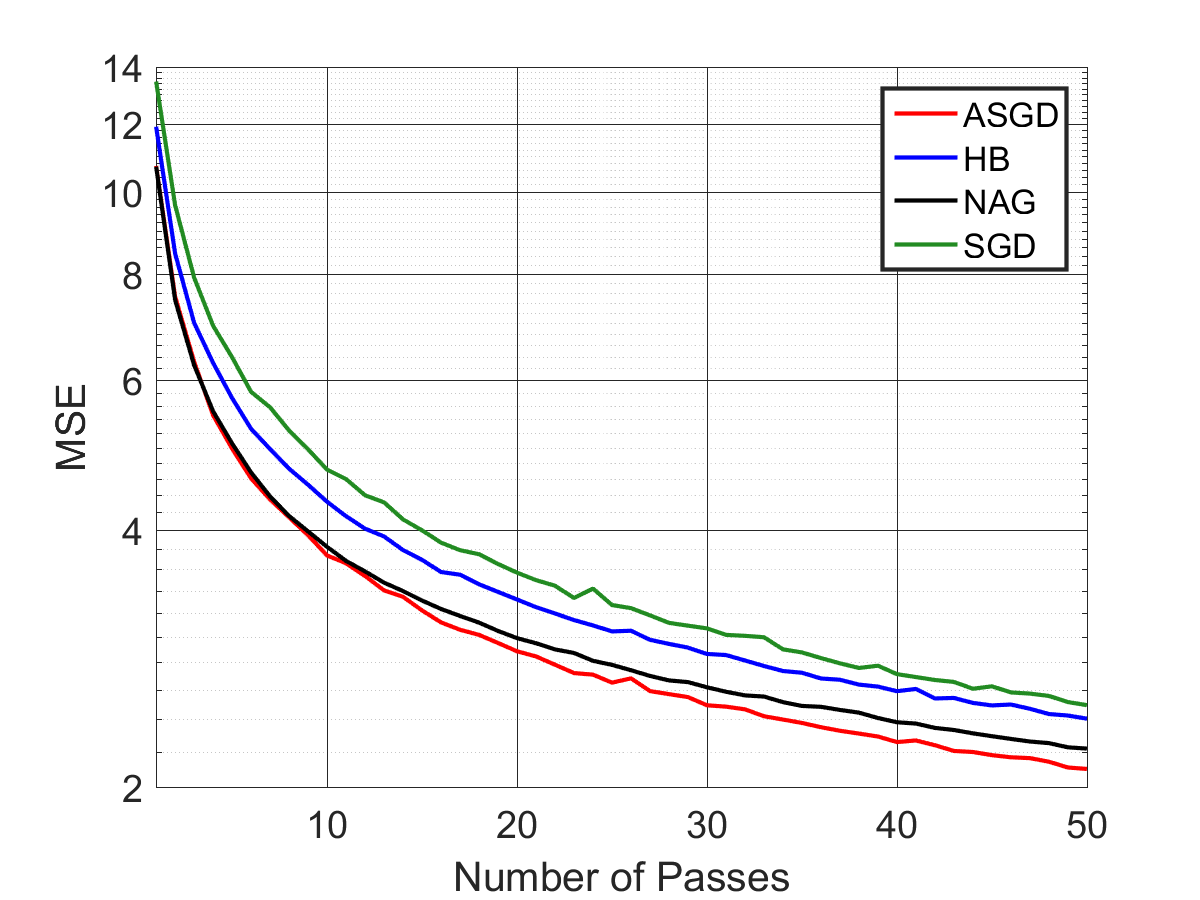}}
	\subfigure{\includegraphics[width=0.45\columnwidth]{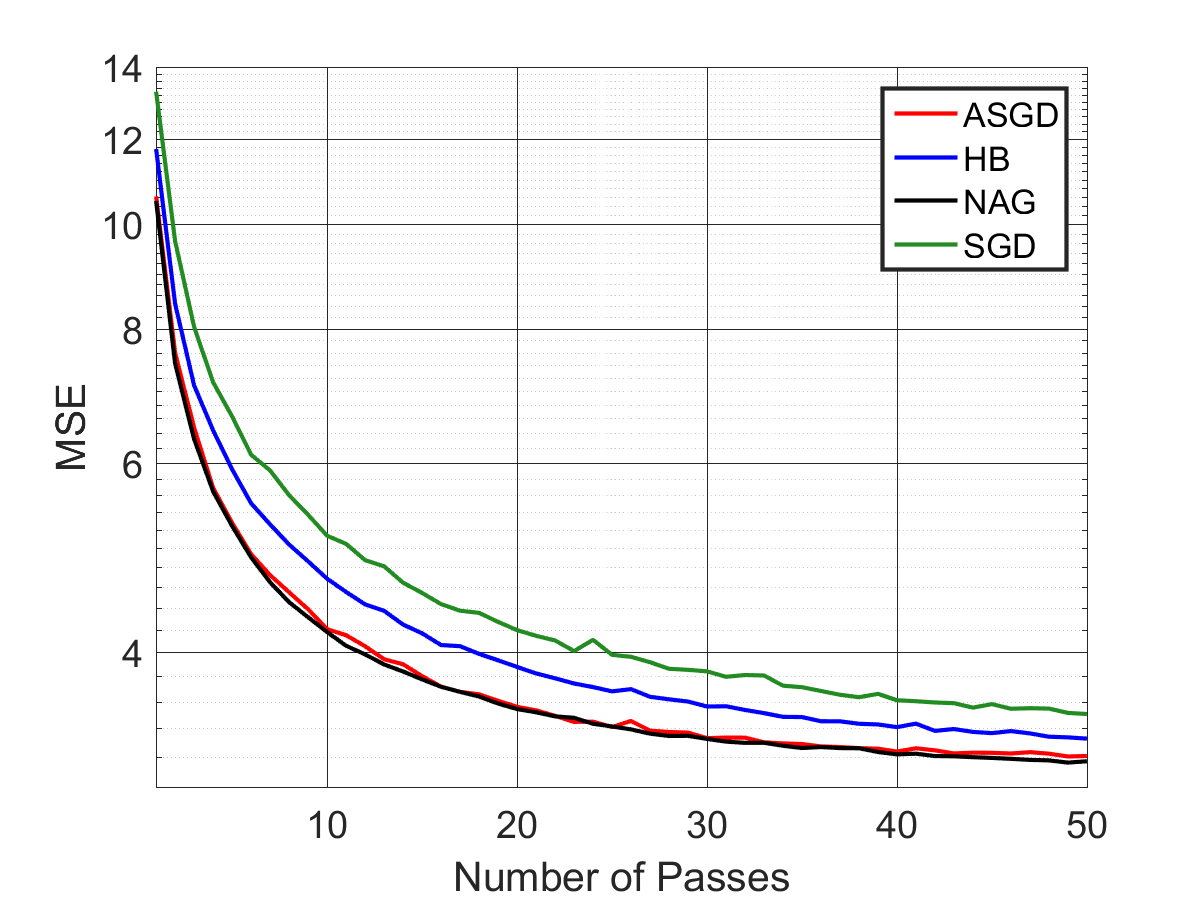}}
	\caption{Training loss (left) and test loss (right) while training deep autoencoder for~\mnist~with minibatch size $8$. Clearly, \asgd matches performance of \nag and outperforms \sgd on the test data. \hb also outperforms \sgd.}
	\label{fig:mnist-train-test-bs8}
\end{figure}
\begin{figure}[h!]
	\centering
	\subfigure{\includegraphics[width=0.45\columnwidth]{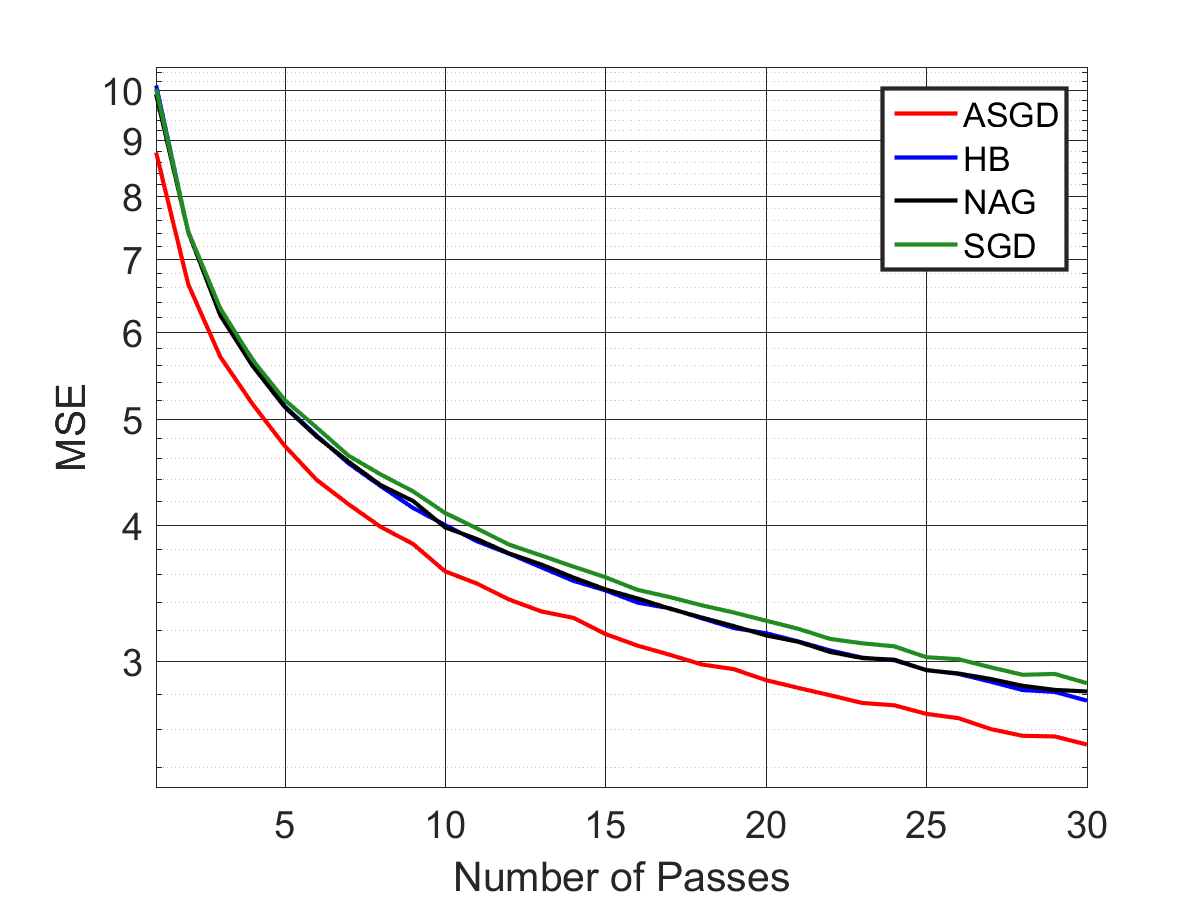}}
	\subfigure{\includegraphics[width=0.45\columnwidth]{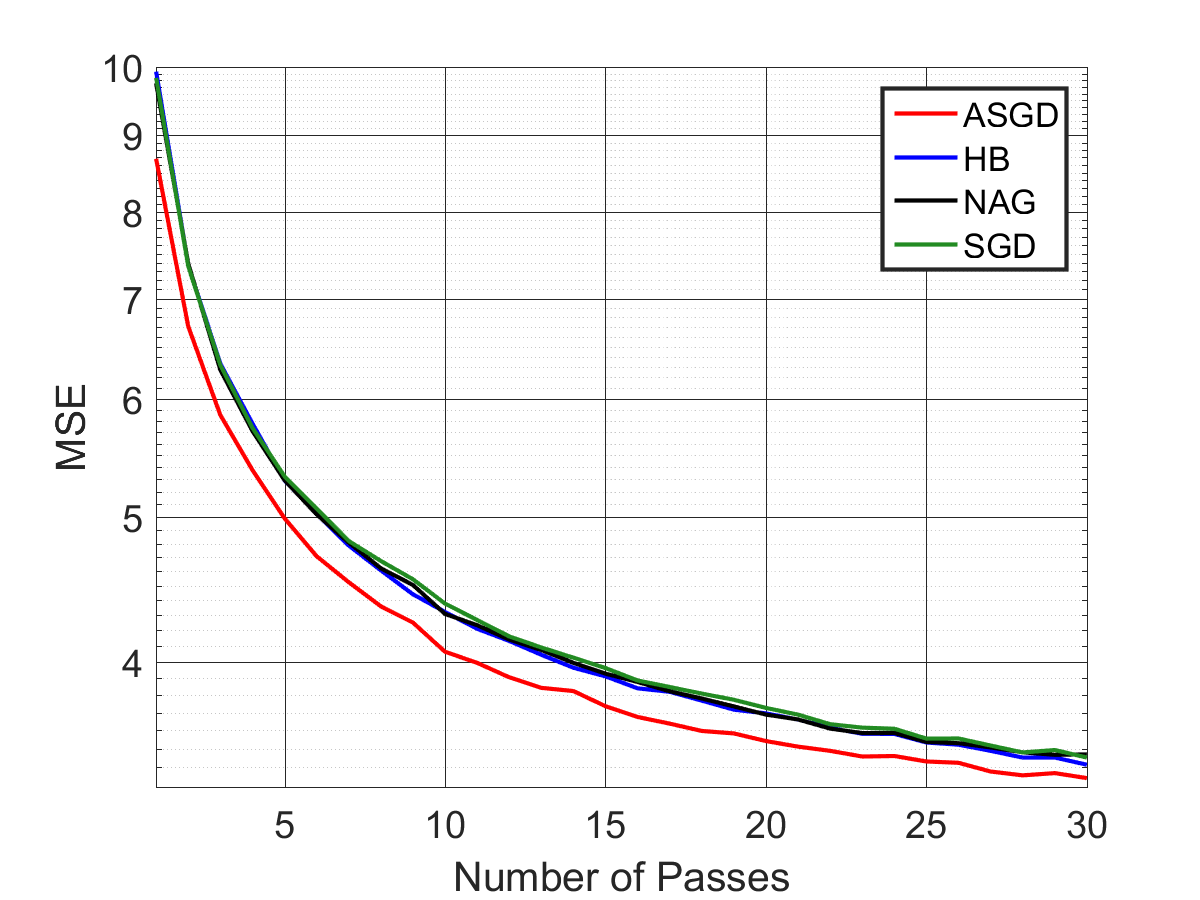}}
	\caption{Training loss (left) and test loss (right) while training deep autoencoder for~\mnist~with minibatch size $1$. Interestingly, ~\sgd,~\hb~and~\nag, all decrease the loss at a similar rate, while~\asgd~decays at a faster rate.}
	\label{fig:mnist-train-test-bs1}
\end{figure}

\textbf{Effect of minibatch sizes}: While~\hb~and~\nag~decay the loss faster compared to~\sgd~for a minibatch size of $8$ (Figure~\ref{fig:mnist-train-test-bs8}), this superior decay rate does not hold for a minibatch size of $1$ (Figure~\ref{fig:mnist-train-test-bs1}). This supports our intuitions from the stochastic linear regression setting, where we demonstrate that~\hb~and~\nag~are suboptimal in the~\sfo~model.

\textbf{Comparison of~\asgd~with momentum methods}: While~\asgd~performs slightly better than \nag for batch size $8$ in the training error (Figure~\ref{fig:mnist-train-test-bs8}),~\asgd~decays the error at a faster rate compared to all the three other methods for a batch size of $1$ (Figure~\ref{fig:mnist-train-test-bs1}).

\subsection{Deep Residual Networks for CIFAR-10}\label{sec:cifar}
We will now present experimental results on training deep residual networks~\citep{kaiminghe2016resnet} with pre-activation blocks~\citet{kaiminghe2016identity} for classifying images in~\cifar~\citep{krizhevsky2009learning}; the network we use has $44$ layers (dubbed preresnet-44). The code for this section was downloaded from~\citet{preresenet44}. One of the most distinct characteristics of this experiment compared to our previous experiments is learning rate decay. We use a validation set based decay scheme, wherein, after every $3$ epochs, we decay the learning rate by a certain factor (which we grid search on) if the validation zero one error does not decrease by at least a certain amount (precise numbers are provided in the appendix since they vary across batch sizes). Due to space constraints, we present only a subset of training error plots.
Please see Appendix~\ref{app:cifar-det} for some more plots on training errors.

\textbf{Effect of minibatch sizes}:
Our first experiment tries to understand how the performance of~\hb~and~\nag compare with that of~\sgd~and how it varies with minibatch sizes. Figure~\ref{fig:cifar-test-bs} presents the test zero one error for minibatch sizes of $8$ and $128$. While training with batch size $8$ was done for $40$ epochs, with batch size $128$, it was done for $120$ epochs. We perform a grid search over all parameters for each of these algorithms. See Appendix~\ref{app:cifar-det} for details on the grid search parameters. We observe that final error achieved by~\sgd,~\hb~and~\nag~are all very close for both batch sizes. While~\nag~exhibits a superior rate of convergence compared to~\sgd~and~\hb~for batch size $128$, this superior rate of convergence disappears for a batch size of $8$.

\begin{figure}[t!]
	\centering
	\subfigure{\includegraphics[width=0.32\columnwidth]{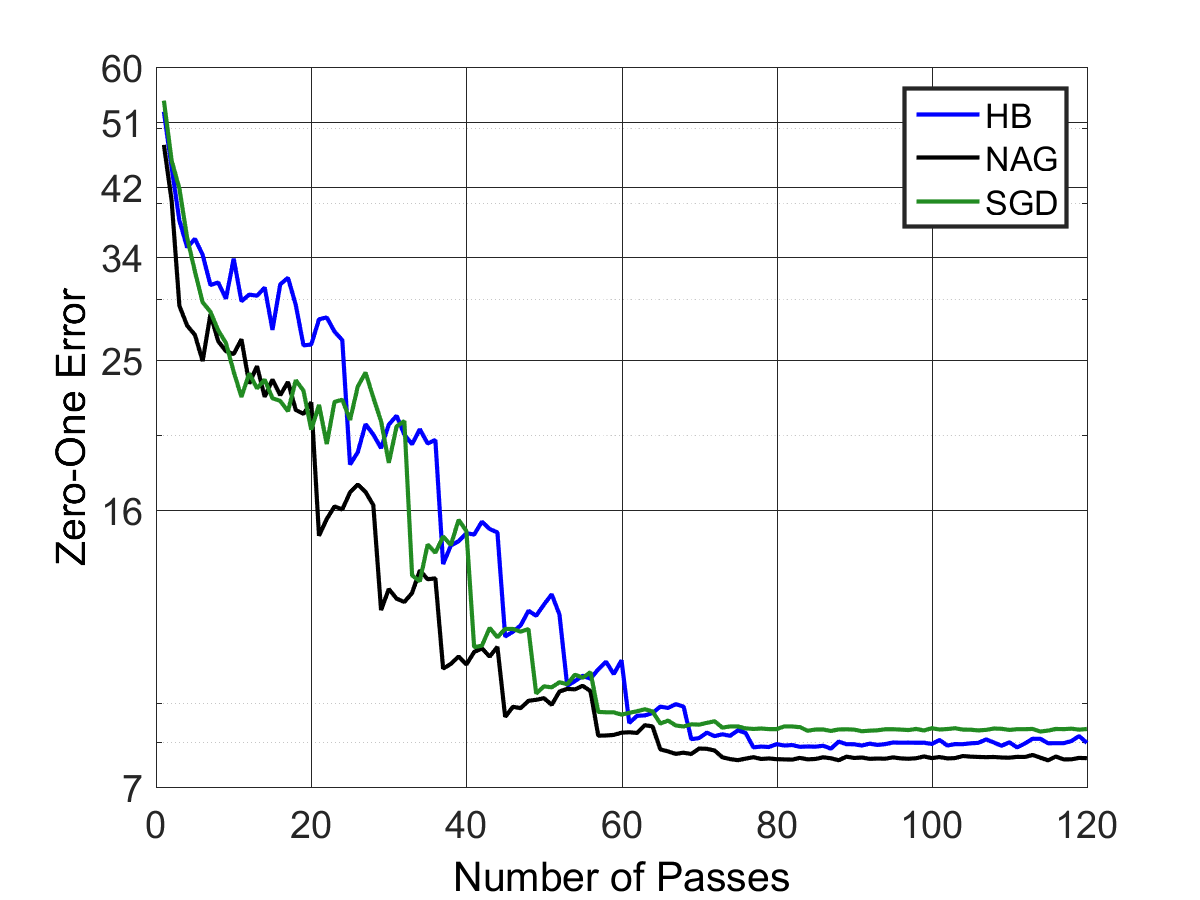}}
	\subfigure{\includegraphics[width=0.32\columnwidth]{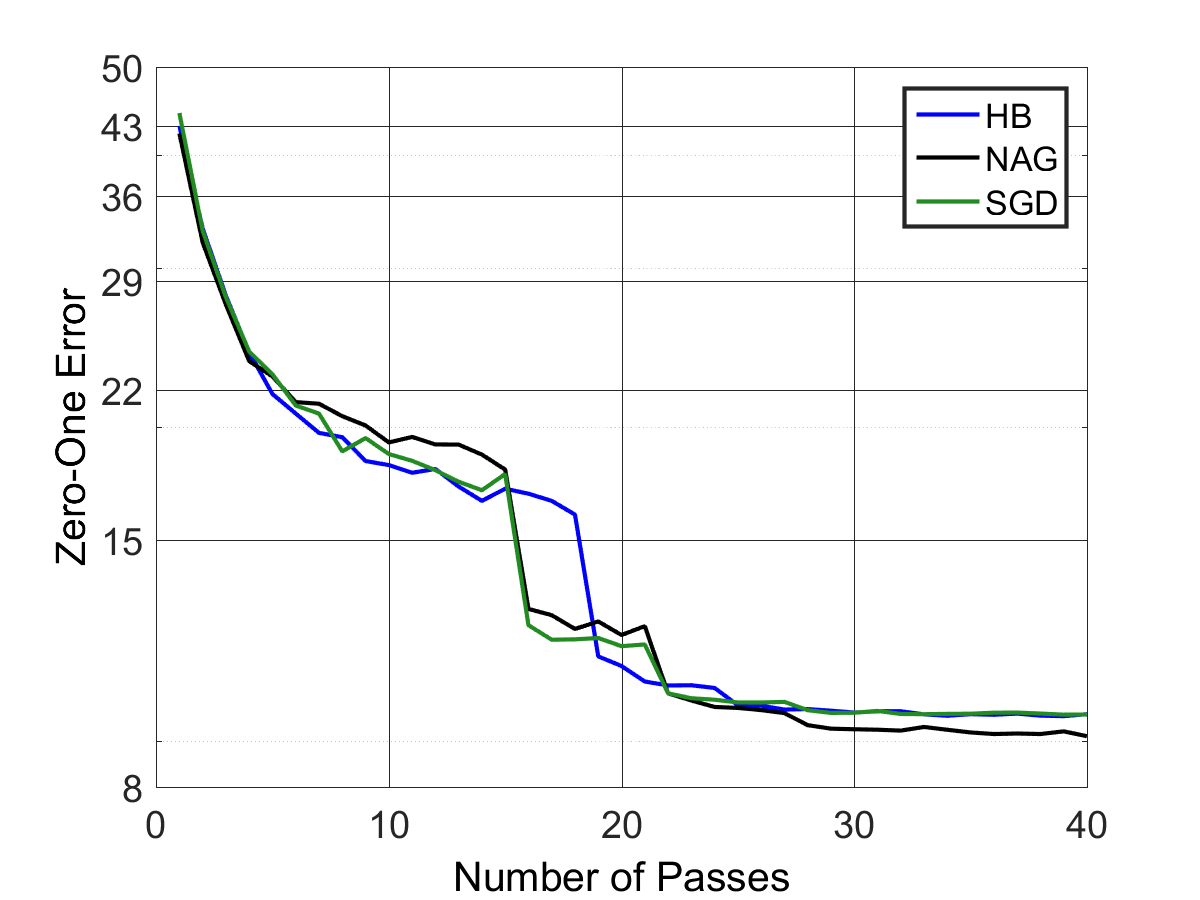}}
	\subfigure{\includegraphics[width=0.32\columnwidth]{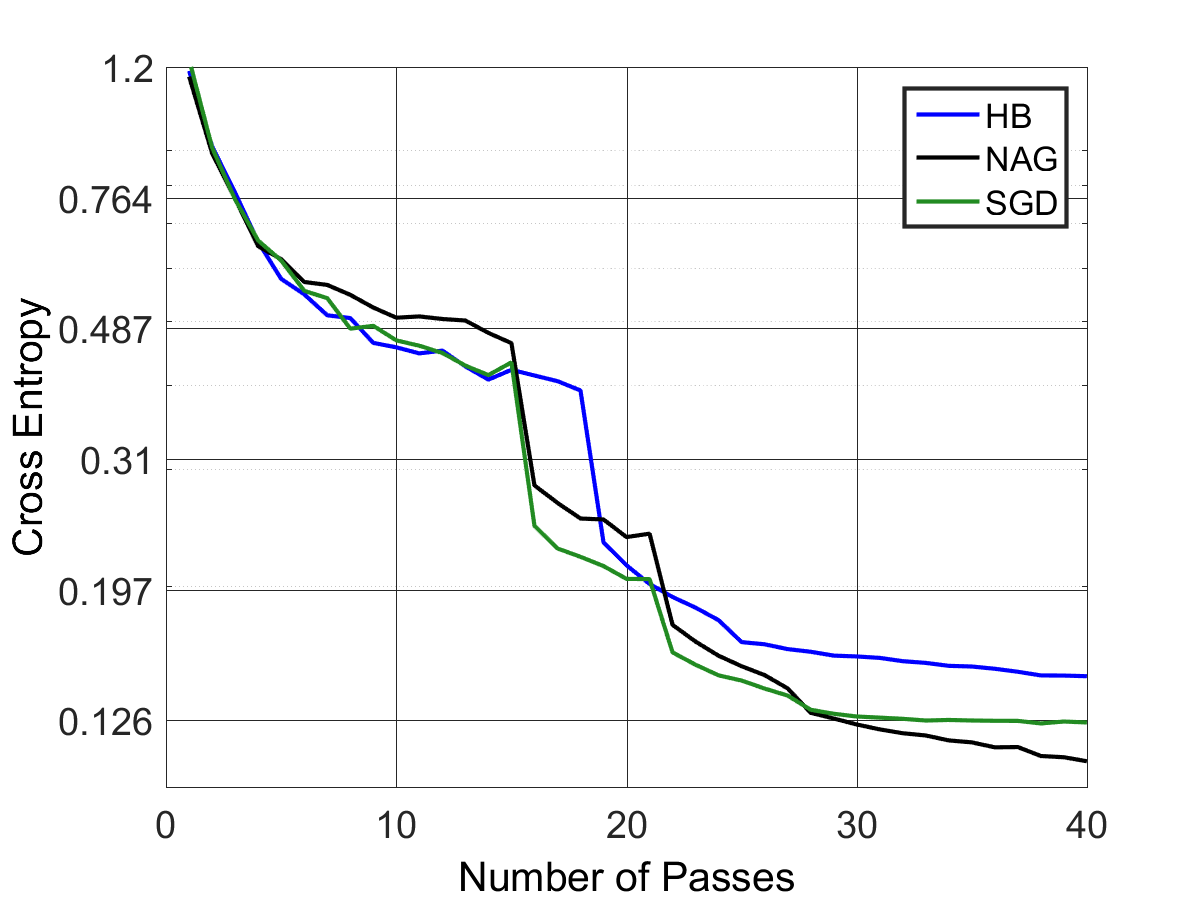}}
	\caption{Test zero one loss for batch size $128$ (left), batch size $8$ (center) and training function value for batch size $8$ (right) for~\sgd,~\hb~and~\nag.}
	\label{fig:cifar-test-bs}
\end{figure}

\textbf{Comparison of~\asgd~with momentum methods}: The next experiment tries to understand how~\asgd compares with~\hb~and~\nag. The errors achieved by various methods when we do grid search over all parameters are presented in Table~\ref{table:cifar}. Note that the final test errors for batch size $128$ are better than those for batch size $8$ since the former was run for $120$ epochs while the latter was run only for $40$ epochs (due to time constraints).

\begin{table}[h!]
	\centering
	\begin{tabular}{|c | c | c|}
		\hline
		Algorithm & Final test error -- batch size $128$ & Final test error -- batch size $8$ \\\hhline{|=|=|=|} 
		\sgd & $8.32\pm0.21$ & $9.57\pm0.18$ \\ 
		\hb & $7.98\pm0.19$ & $9.28\pm0.25$ \\
		\nag & $7.63\pm0.18$ & $9.07\pm0.18$ \\
		\asgd & $\bf{7.23\pm0.22}$  & $\bf{8.52\pm0.16}$ \\ [1ex] 
		\hline
	\end{tabular}
	\caption{Final test errors achieved by various methods for batch sizes of $128$ and $8$. The hyperparameters have been chosen by grid search.}
	\label{table:cifar}
\end{table}

While the final error achieved by~\asgd~is similar/favorable compared to all other methods, we are also interested in understanding whether~\asgd~has a superior convergence speed. For this experiment, we need to address the issue of differing learning rates used by various algorithms and different iterations where they decay learning rates. So, for each of~\hb~and~\nag, we choose the learning rate and decay factors by grid search, use these values for~\asgd and do grid search only over long step parameter $\kappa$ and momentum $\alpha$ for \asgd. The results are presented in Figures~\ref{fig:cifar-test-hb-asgd} and~\ref{fig:cifar-test-nag-asgd}. For batch size $128$,~\asgd decays error at a faster rate compared to both~\hb~and~\nag. For batch size $8$, while we see a superior convergence of~\asgd~compared to~\nag, we do not see this superiority over~\hb. The reason for this turns out to be that the learning rate for~\hb, which we also use for~\asgd, turns out to be quite suboptimal for~\asgd. So, for batch size $8$, we also compare fully optimized (i.e., grid search over learning rate as well)~\asgd~with~\hb. The superiority of~\asgd~over~\hb~is clear from this comparison. These results suggest that~\asgd~decays error at a faster rate compared to~\hb~and~\nag~across different batch sizes.
\begin{figure}[t!]
	\centering
	\subfigure{\includegraphics[width=0.32\columnwidth]{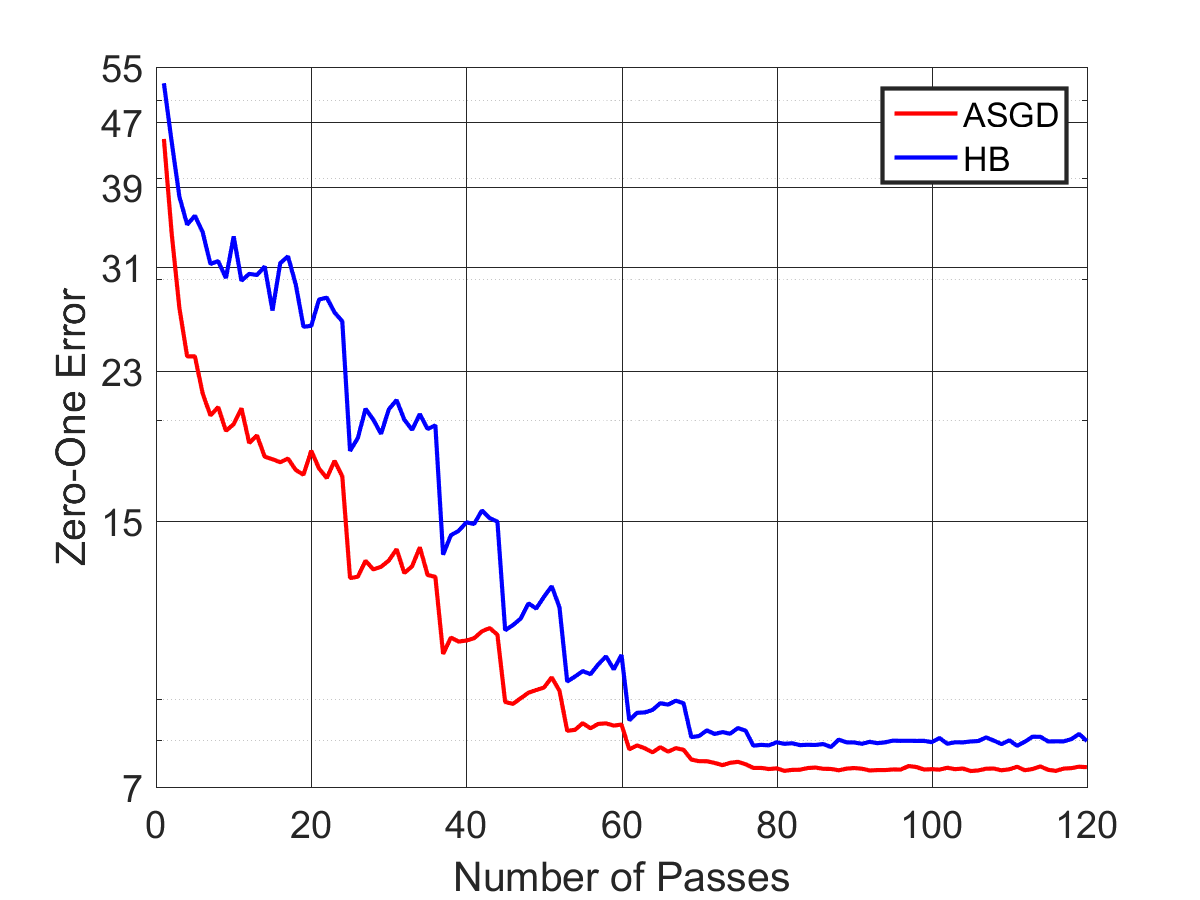}}
	\subfigure{\includegraphics[width=0.32\columnwidth]{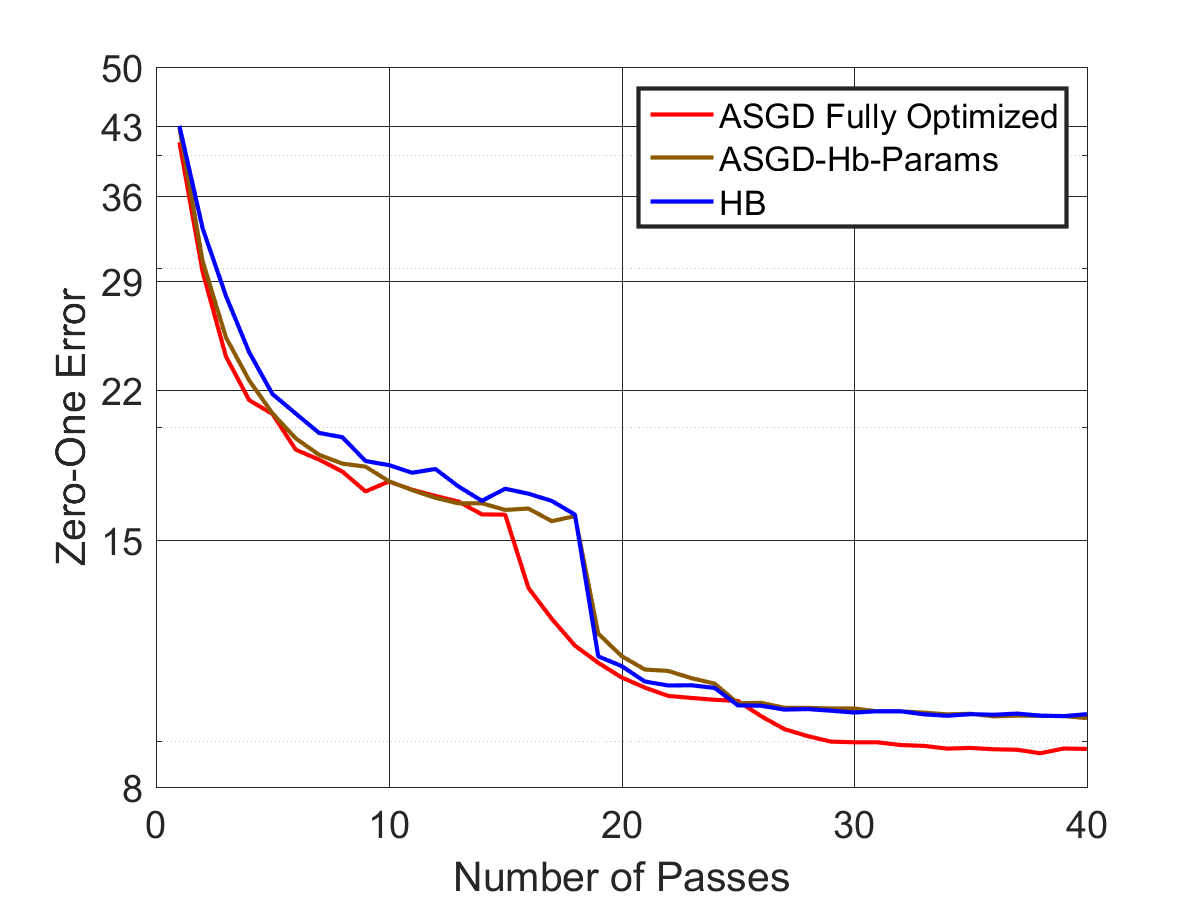}}
	\subfigure{\includegraphics[width=0.32\columnwidth]{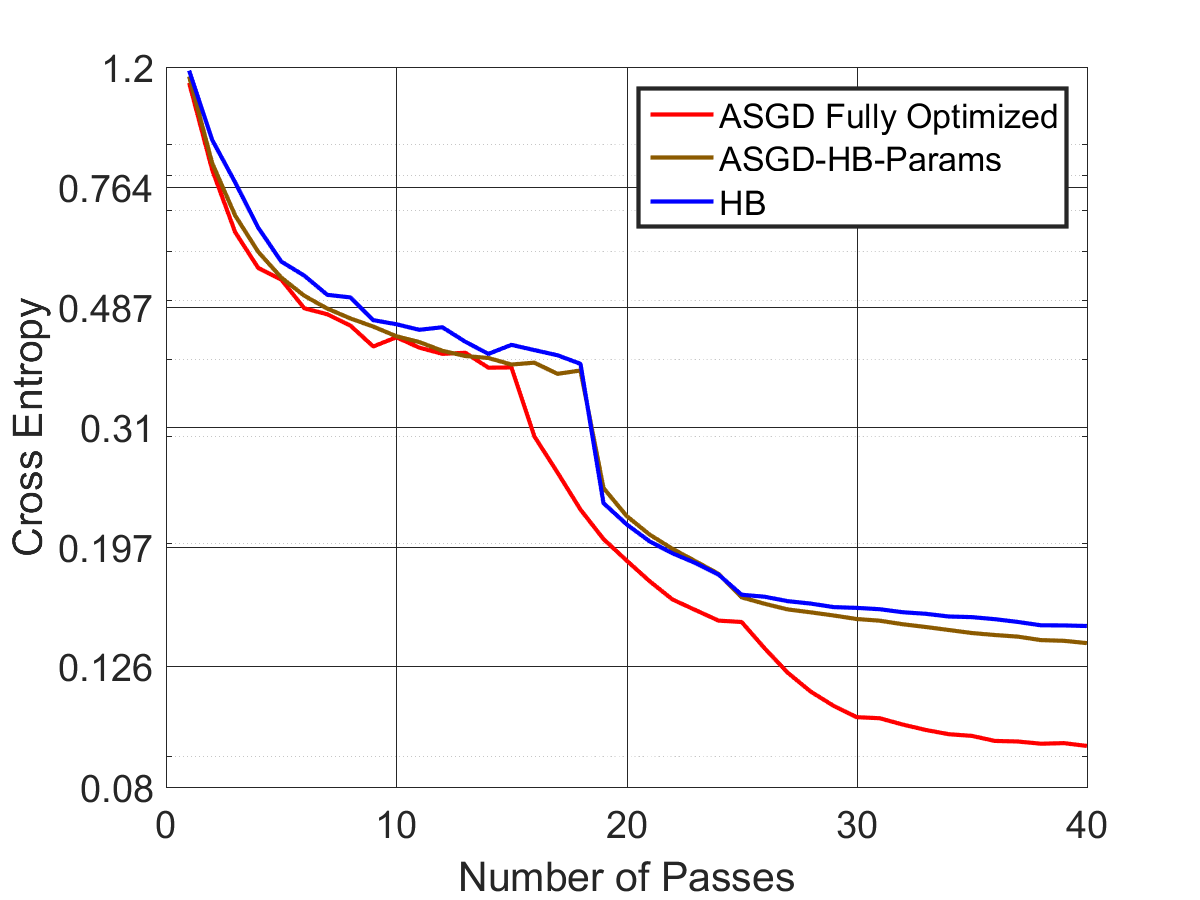}}
	\caption{Test zero one loss for batch size $128$ (left), batch size $8$ (center) and training function value for batch size $8$ (right) for~\asgd~compared to~\hb. In the above plots, both~\asgd~and~\asgd-Hb-Params refer to~\asgd~run with the learning rate and decay schedule of~\hb.~\asgd-Fully-Optimized refers to~\asgd~where learning rate and decay schedule were also selected by grid search.}
	\label{fig:cifar-test-hb-asgd}
\end{figure}
\begin{figure}[t!]
	\centering
	\subfigure{\includegraphics[width=0.32\columnwidth]{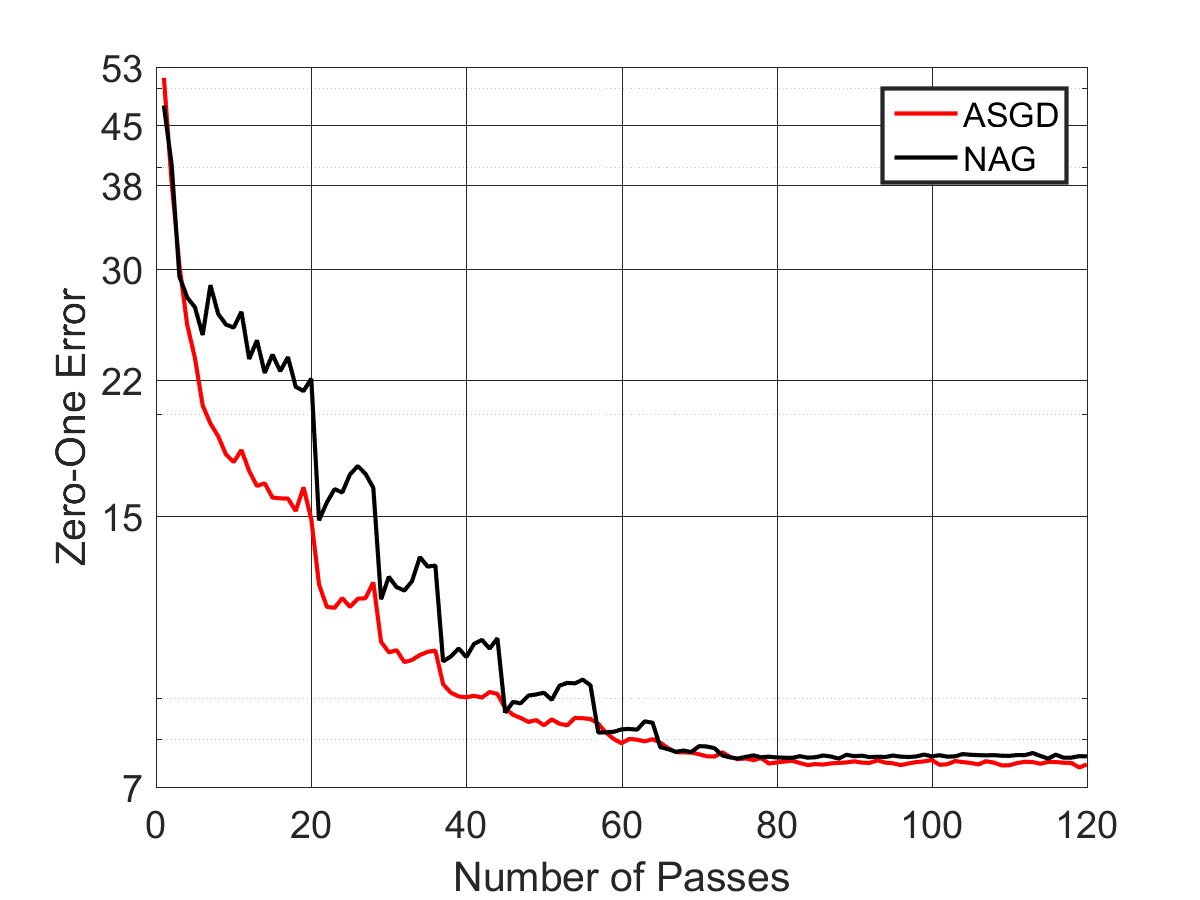}}
	\subfigure{\includegraphics[width=0.32\columnwidth]{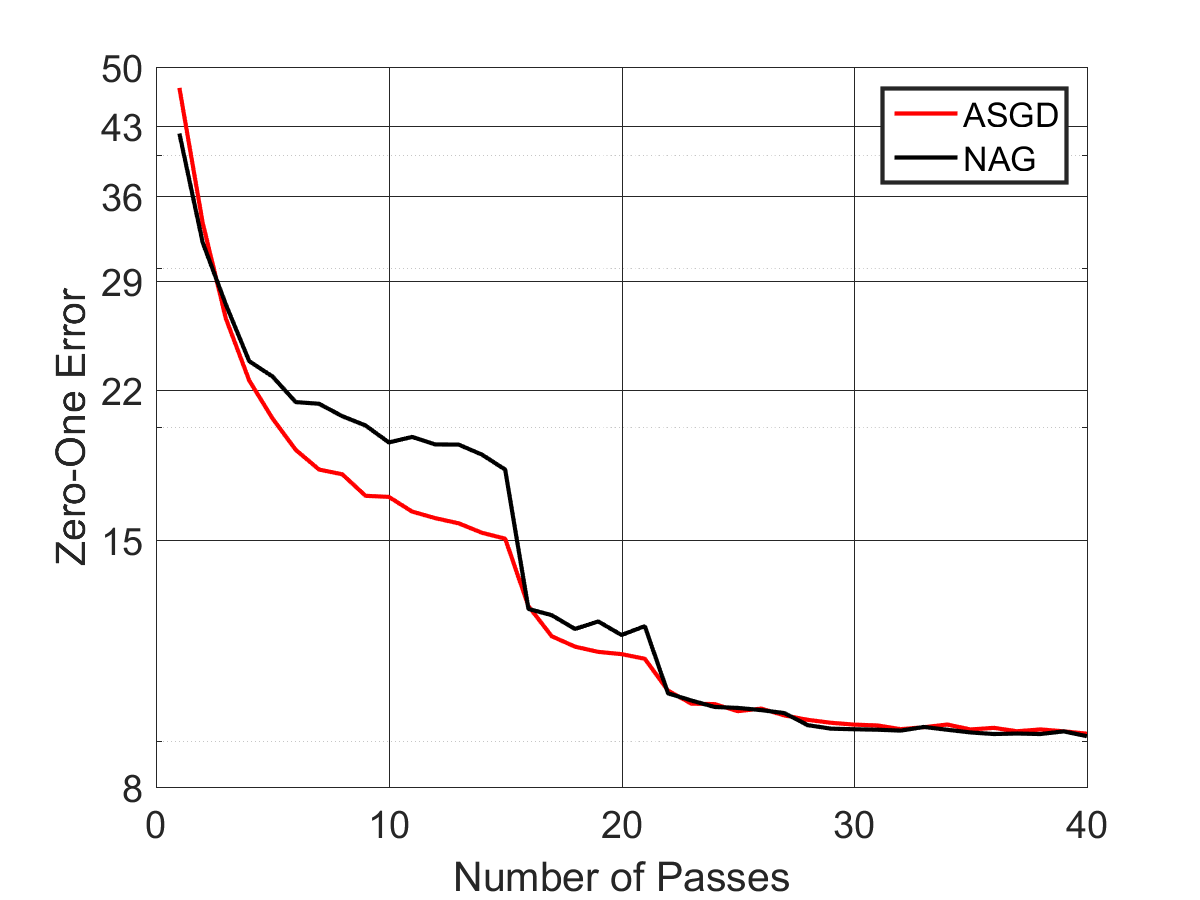}}
	\subfigure{\includegraphics[width=0.32\columnwidth]{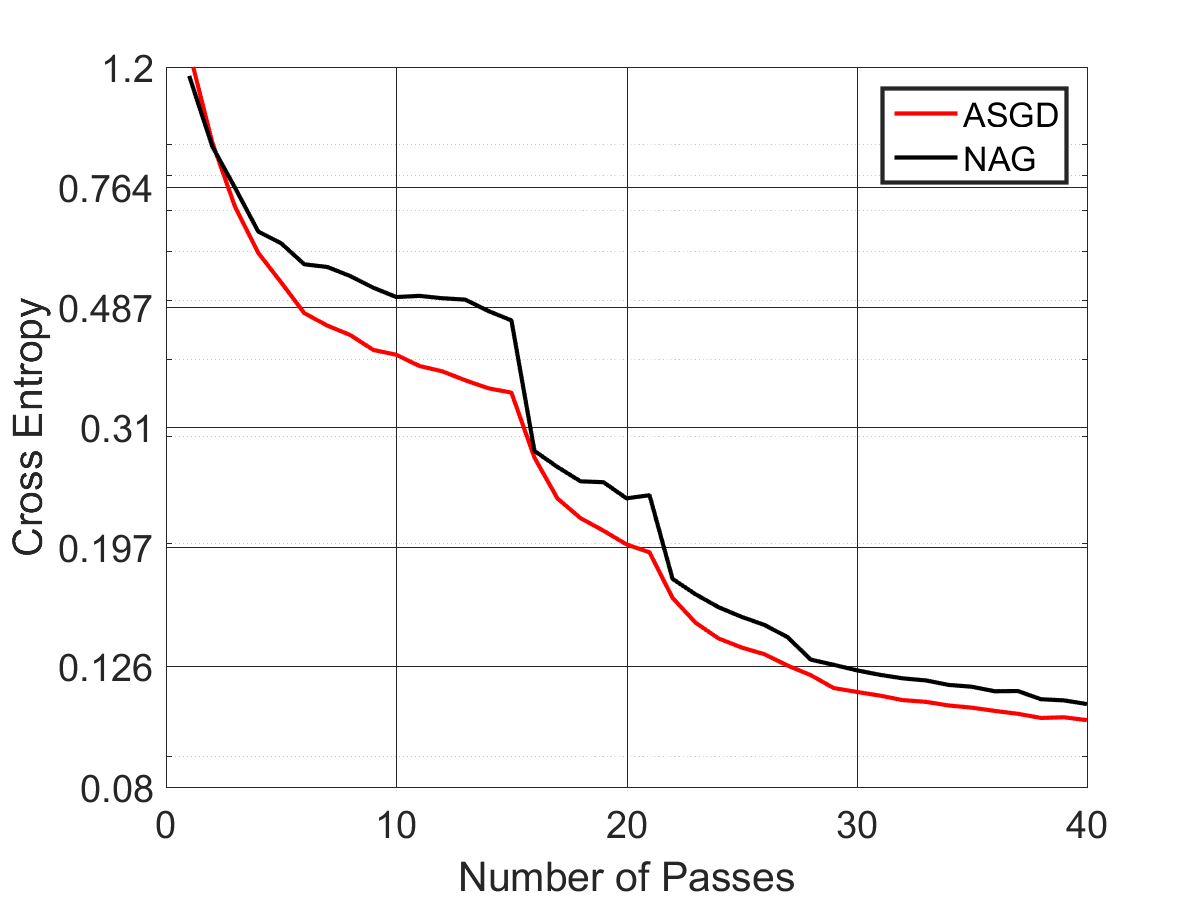}}
	\caption{Test zero one loss for batch size $128$ (left), batch size $8$ (center) and training function value for batch size $8$ (right) for~\asgd~compared to~\nag. In the above plots,~\asgd~was run with the learning rate and decay schedule of~\nag. Other parameters were selected by grid search.}
	\label{fig:cifar-test-nag-asgd}
\end{figure}

\section{Related Work}\label{sec:related}
{\bf First order oracle methods}: The primary method in this family is Gradient Descent (GD)~\citep{cauchy1847}. As mentioned previously, GD is suboptimal for smooth convex optimization~\citep{Nesterov04}, and this is addressed using momentum methods such as the Heavy Ball method~\citep{polyak1964some} (for quadratics), and Nesterov's Accelerated gradient descent~\citep{nesterov1983method}.

{\bf Stochastic first order methods and noise stability}: The simplest method employing the \sfoa is \sgd~\citep{RobbinsM51}; the effectiveness of \sgd has been immense, and its applicability goes well beyond optimizing convex objectives. Accelerating \sgd is a tricky proposition given the instability of fast gradient methods in dealing with noise, as evidenced by several negative results which consider statistical~\citep{Proakis74,Polyak87,RoyS90}, numerical~\citep{Paige71,Greenbaum89} and adversarial errors~\citep{dAspremont08,DevolderGN14}. A result of~\citet{jain2017accelerating} developed the first provably accelerated \sgd method for linear regression which achieved minimax rates, inspired by a method of~\citet{Nesterov12}. Schemes of~\citet{ghadimi2012optimal,ghadimi2013optimal,DieuleveutFB16}, which indicate acceleration is possible with noisy gradients do not hold in the \sfoa model satisfied by algorithms that are run in practice (see~\citet{jain2017accelerating} for more details). 

While \hb~\citep{polyak1964some} and \nag~\citep{nesterov1983method} are known to be effective in case of exact first order oracle, for the~\sfoa, the theoretical performance of \hb and \nag is not well understood.

{\bf Understanding Stochastic Heavy Ball:} Understanding \hb's performance with inexact gradients has been considered in efforts spanning several decades, in many communities like controls, optimization and signal processing.~\citet{Polyak87} considered \hb with noisy gradients and concluded that the improvements offered by \hb with inexact gradients vanish unless strong assumptions on the inexactness was considered; an instance of this is when the variance of inexactness decreased as the iterates approach the minimizer. \citet{Proakis74,RoyS90,SharmaSB98} suggest that the improved non-asymptotic rates offered by stochastic \hb arose at the cost of worse asymptotic behavior. We resolve these unquantified improvements on rates as being just constant factors over \sgd, in stark contrast to the gains offered by \asgd.~\citet{LoizouR17} state their method as Stochastic \hb but require stochastic gradients that nearly behave as exact gradients; indeed, their rates match that of the standard \hb method~\citep{polyak1964some}. Such rates are not information theoretically possible (see \citet{jain2017accelerating}), especially with a batch size of $1$ or even with constant sized minibatches.

{\bf Accelerated and Fast Methods for finite-sums:} There have been developments pertaining to faster methods for finite-sums (also known as offline stochastic optimization): amongst these are methods such as SDCA~\citep{ShwartzZ12}, SAG~\citep{RouxSB12}, SVRG~\citep{JohnsonZ13}, SAGA~\citep{DefazioBJ14}, which offer linear convergence rates for strongly convex finite-sums, improving over SGD's sub-linear rates~\citep{RakhlinSS12}. These methods have been improved using accelerated variants~\citep{ShwartzZ14,FrostigGKS15b,LinMH15,Defazio16,Zhu16}. Note that these methods require storing the entire training set in memory and taking multiple passes over the same for guaranteed progress. Furthermore, these methods require computing a batch gradient or require  memory requirements (typically $\Omega(|\text{ training data points}|)$). For deep learning problems, data augmentation is often deemed necessary for achieving good performance; this implies computing quantities such as batch gradient (or storage necessities) over this augmented dataset is often infeasible. Such requirements are mitigated by the use of simple streaming methods such as \sgd, \asgd, \hb, \nag. For other technical distinctions between the offline and online stochastic methods refer to \citet{FrostigGKS15}.

{\bf Practical methods for training deep networks}: Momentum based methods employed with stochastic gradients~\citep{sutskever2013importance} have become standard and popular in practice. These schemes tend to outperform standard \sgd on several important practical problems. As previously mentioned, we attribute this improvement to effect of mini-batching rather than improvement offered by \hb or \nag in the \sfoa model. Schemes such as Adagrad~\citep{Duchi2011Adagrad}, RMSProp~\citep{Hinton2012RMS}, Adam~\citep{Kingma2014Adam} represent an important and useful class of algorithms. The advantages offered by these methods are orthogonal to the advantages offered by fast gradient methods; it is an important direction to explore augmenting these methods with \asgd as opposed to standard \hb or \nag based acceleration schemes.

\citet{Chaudhari17entropic} proposed Entropy-SGD, which is an altered objective that adds a local strong convexity term to the actual empirical risk objective, with an aim to improve generalization. However, we do not understand convergence rates for convex problems or the generalization ability of this technique in a rigorous manner.~\citet{Chaudhari17entropic} propose to use SGD in their procedure but mention that they employ the \hb/\nag method in their implementation for achieving better performance. Naturally, we can use \asgd in this context. Path normalized SGD~\citep{NeyshaburSS15} is a variant of \sgd that alters the metric on which the weights are optimized. As noted in their paper, path normalized \sgd could be improved using \hb/\nag (or even the \asgd method).

\section{Conclusions and Future Directions}\label{sec:disc}
In this paper, we show that the performance gain of \hb over \sgd in stochastic setting is attributed to mini-batching rather than the algorithm's ability to {\em accelerate} with stochastic gradients. Concretely, we provide a formal proof that for several {\em easy} problem instances, \hb does not outperform \sgd despite large condition number of the problem; we observe this trend for \nag in our experiments. In contrast, \asgd~\citep{jain2017accelerating} provides significant improvement over \sgd for these problem instances. We observe similar trends when training a resnet on \cifar and an autoencoder on \mnist. This work motivates several directions such as understanding the behavior of \asgd on domains such as NLP, and developing automatic momentum tuning schemes~\citep{Zhang2017yellowfin}.

\subsubsection*{Acknowledgments}
Sham Kakade acknowledges funding from Washington Research Foundation Fund for Innovation in Data-Intensive Discovery and the NSF through awards CCF-$1637360$, CCF-$1703574$ and CCF-$1740551$.


\pagebreak

\begin{thebibliography}{49}
\providecommand{\natexlab}[1]{#1}
\providecommand{\url}[1]{\texttt{#1}}
\expandafter\ifx\csname urlstyle\endcsname\relax
  \providecommand{\doi}[1]{doi: #1}\else
  \providecommand{\doi}{doi: \begingroup \urlstyle{rm}\Url}\fi

\bibitem[Allen-Zhu(2016)]{Zhu16}
Z.~Allen-Zhu.
\newblock Katyusha: The first direct acceleration of stochastic gradient
  methods.
\newblock \emph{CoRR}, abs/1603.05953, 2016.

\bibitem[Bottou and Bousquet(2007)]{BottouB07}
L.~Bottou and O.~Bousquet.
\newblock The tradeoffs of large scale learning.
\newblock In \emph{NIPS 20}, 2007.

\bibitem[Cauchy(1847)]{cauchy1847}
L.~A. Cauchy.
\newblock M\'ethode g\'en\'erale pour la r\'esolution des syst\'emes
  d'\'equations simultanees.
\newblock \emph{C. R. Acad. Sci. Paris}, 1847.

\bibitem[Chaudhari et~al.(2017)Chaudhari, Choromanska, Soatto, LeCun, Baldassi,
  Borgs, Chayes, Sagun, and Zecchina]{Chaudhari17entropic}
P.~Chaudhari, A.~Choromanska, S.~Soatto, Y.~LeCun, C.~Baldassi, C.~Borgs,
  J.~Chayes, L.~Sagun, and R.~Zecchina.
\newblock Entropy-sgd: Biasing gradient descent into wide valleys.
\newblock \emph{CoRR}, abs/1611.01838, 2017.

\bibitem[d'Aspremont(2008)]{dAspremont08}
A.~d'Aspremont.
\newblock Smooth optimization with approximate gradient.
\newblock \emph{SIAM Journal on Optimization}, 19\penalty0 (3):\penalty0
  1171--1183, 2008.

\bibitem[Defazio(2016)]{Defazio16}
A.~Defazio.
\newblock A simple practical accelerated method for finite sums.
\newblock \emph{Advances in Neural Information Processing Systems 29 (NIPS
  2016)}, 2016.

\bibitem[Defazio et~al.(2014)Defazio, Bach, and Lacoste-Julien]{DefazioBJ14}
A.~Defazio, F.~R. Bach, and S.~Lacoste-Julien.
\newblock {SAGA}: {A} fast incremental gradient method with support for
  non-strongly convex composite objectives.
\newblock In \emph{NIPS 27}, 2014.

\bibitem[Devolder et~al.(2014)Devolder, Glineur, and Nesterov]{DevolderGN14}
O.~Devolder, F.~Glineur, and Y.~E. Nesterov.
\newblock First-order methods of smooth convex optimization with inexact
  oracle.
\newblock \emph{Mathematical Programming}, 146:\penalty0 37--75, 2014.

\bibitem[Dieuleveut et~al.(2016)Dieuleveut, Flammarion, and
  Bach]{DieuleveutFB16}
A.~Dieuleveut, N.~Flammarion, and F.~R. Bach.
\newblock Harder, better, faster, stronger convergence rates for least-squares
  regression.
\newblock \emph{CoRR}, abs/1602.05419, 2016.

\bibitem[Duchi et~al.(2011)Duchi, Hazan, and Singer]{Duchi2011Adagrad}
J.~C. Duchi, E.~Hazan, and Y.~Singer.
\newblock Adaptive subgradient methods for online learning and stochastic
  optimization.
\newblock \emph{Journal of Machine Learning Research}, 12:\penalty0 2121--2159,
  2011.

\bibitem[Frostig et~al.(2015{\natexlab{a}})Frostig, Ge, Kakade, and
  Sidford]{FrostigGKS15b}
R.~Frostig, R.~Ge, S.~Kakade, and A.~Sidford.
\newblock Un-regularizing: approximate proximal point and faster stochastic
  algorithms for empirical risk minimization.
\newblock In \emph{ICML}, 2015{\natexlab{a}}.

\bibitem[Frostig et~al.(2015{\natexlab{b}})Frostig, Ge, Kakade, and
  Sidford]{FrostigGKS15}
R.~Frostig, R.~Ge, S.~M. Kakade, and A.~Sidford.
\newblock Competing with the empirical risk minimizer in a single pass.
\newblock In \emph{COLT}, 2015{\natexlab{b}}.

\bibitem[Ghadimi and Lan(2012)]{ghadimi2012optimal}
S.~Ghadimi and G.~Lan.
\newblock Optimal stochastic approximation algorithms for strongly convex
  stochastic composite optimization i: A generic algorithmic framework.
\newblock \emph{SIAM Journal on Optimization}, 2012.

\bibitem[Ghadimi and Lan(2013)]{ghadimi2013optimal}
S.~Ghadimi and G.~Lan.
\newblock Optimal stochastic approximation algorithms for strongly convex
  stochastic composite optimization, ii: shrinking procedures and optimal
  algorithms.
\newblock \emph{SIAM Journal on Optimization}, 2013.

\bibitem[Greenbaum(1989)]{Greenbaum89}
A.~Greenbaum.
\newblock Behavior of slightly perturbed lanczos and conjugate-gradient
  recurrences.
\newblock \emph{Linear Algebra and its Applications}, 1989.

\bibitem[He et~al.(2016{\natexlab{a}})He, Zhang, Ren, and
  Sun]{kaiminghe2016identity}
K.~He, X.~Zhang, S.~Ren, and J.~Sun.
\newblock Identity mappings in deep residual networks.
\newblock In \emph{ECCV (4)}, Lecture Notes in Computer Science, pages
  630--645. Springer, 2016{\natexlab{a}}.

\bibitem[He et~al.(2016{\natexlab{b}})He, Zhang, Ren, and
  Sun]{kaiminghe2016resnet}
K.~He, X.~Zhang, S.~Ren, and J.~Sun.
\newblock Deep residual learning for image recognition.
\newblock In \emph{CVPR}, pages 770--778, 2016{\natexlab{b}}.

\bibitem[Hinton and Salakhutdinov(2006)]{hinton2006reducing}
G.~E. Hinton and R.~R. Salakhutdinov.
\newblock Reducing the dimensionality of data with neural networks.
\newblock \emph{science}, 313\penalty0 (5786):\penalty0 504--507, 2006.

\bibitem[Jain et~al.(2016)Jain, Kakade, Kidambi, Netrapalli, and
  Sidford]{jain2016parallelizing}
P.~Jain, S.~M. Kakade, R.~Kidambi, P.~Netrapalli, and A.~Sidford.
\newblock Parallelizing stochastic approximation through mini-batching and
  tail-averaging.
\newblock \emph{arXiv preprint arXiv:1610.03774}, 2016.

\bibitem[Jain et~al.(2017)Jain, Kakade, Kidambi, Netrapalli, and
  Sidford]{jain2017accelerating}
P.~Jain, S.~M. Kakade, R.~Kidambi, P.~Netrapalli, and A.~Sidford.
\newblock Accelerating stochastic gradient descent.
\newblock \emph{arXiv preprint arXiv:1704.08227}, 2017.

\bibitem[Johnson and Zhang(2013)]{JohnsonZ13}
R.~Johnson and T.~Zhang.
\newblock Accelerating stochastic gradient descent using predictive variance
  reduction.
\newblock In \emph{NIPS 26}, 2013.

\bibitem[Kingma and Ba(2014)]{Kingma2014Adam}
D.~P. Kingma and J.~Ba.
\newblock Adam: A method for stochastic optimization.
\newblock \emph{CoRR}, abs/1412.6980, 2014.

\bibitem[Krizhevsky and Hinton(2009)]{krizhevsky2009learning}
A.~Krizhevsky and G.~Hinton.
\newblock Learning multiple layers of features from tiny images.
\newblock 2009.

\bibitem[Lin et~al.(2015)Lin, Mairal, and Harchaoui]{LinMH15}
H.~Lin, J.~Mairal, and Z.~Harchaoui.
\newblock A universal catalyst for first-order optimization.
\newblock In \emph{NIPS}, 2015.

\bibitem[Loizou and Richt{\'{a}}rik(2017)]{LoizouR17}
N.~Loizou and P.~Richt{\'{a}}rik.
\newblock Linearly convergent stochastic heavy ball method for minimizing
  generalization error.
\newblock 2017.

\bibitem[Martens(2010)]{martens2010}
J.~Martens.
\newblock Deep learning via hessian-free optimization.
\newblock In \emph{International conference on machine learning}, 2010.

\bibitem[Martens and Grosse(2015)]{martens2015kronecker}
J.~Martens and R.~Grosse.
\newblock Optimizing neural networks with kronecker-factored approximate
  curvature.
\newblock In \emph{International conference on machine learning}, 2015.

\bibitem[Nesterov(1983)]{nesterov1983method}
Y.~Nesterov.
\newblock A method of solving a convex programming problem with convergence
  rate o (1/k2).
\newblock In \emph{Soviet Mathematics Doklady}, volume~27, pages 372--376,
  1983.

\bibitem[Nesterov(2012{\natexlab{a}})]{Nesterov12b}
Y.~Nesterov.
\newblock Gradient methods for minimizing composite functions.
\newblock \emph{Mathematical Programming Series B}, 2012{\natexlab{a}}.

\bibitem[Nesterov(2004)]{Nesterov04}
Y.~E. Nesterov.
\newblock \emph{Introductory lectures on convex optimization: A basic course},
  volume~87 of \emph{Applied Optimization}.
\newblock Kluwer Academic Publishers, 2004.

\bibitem[Nesterov(2012{\natexlab{b}})]{Nesterov12}
Y.~E. Nesterov.
\newblock Efficiency of coordinate descent methods on huge-scale optimization
  problems.
\newblock \emph{SIAM Journal on Optimization}, 22\penalty0 (2):\penalty0
  341--362, 2012{\natexlab{b}}.

\bibitem[Neyshabur et~al.(2015)Neyshabur, Salakhutdinov, and
  Srebro]{NeyshaburSS15}
B.~Neyshabur, R.~Salakhutdinov, and N.~Srebro.
\newblock Path-sgd: Path-normalized optimization in deep neural networks.
\newblock \emph{CoRR}, abs/1506.02617, 2015.

\bibitem[Paige(1971)]{Paige71}
C.~C. Paige.
\newblock The computation of eigenvalues and eigenvectors of very large sparse
  matrices.
\newblock \emph{PhD Thesis, University of London}, 1971.

\bibitem[Polyak(1964)]{polyak1964some}
B.~T. Polyak.
\newblock Some methods of speeding up the convergence of iteration methods.
\newblock \emph{USSR Computational Mathematics and Mathematical Physics},
  4\penalty0 (5):\penalty0 1--17, 1964.

\bibitem[Polyak(1987)]{Polyak87}
B.~T. Polyak.
\newblock \emph{Introduction to Optimization}.
\newblock Optimization Software, 1987.

\bibitem[preresnet(2017)]{preresenet44}
preresnet.
\newblock Preresnet-44 for cifar-10.
\newblock \url{https://github.com/D-X-Y/ResNeXt-DenseNet}, 2017.
\newblock Accessed: 2017-10-25.

\bibitem[Proakis(1974)]{Proakis74}
J.~G. Proakis.
\newblock Channel identification for high speed digital communications.
\newblock \emph{IEEE Transactions on Automatic Control}, 1974.

\bibitem[pytorch(2017)]{pytorch}
pytorch.
\newblock Pytorch.
\newblock \url{https://github.com/pytorch}, 2017.
\newblock Accessed: 2017-10-25.

\bibitem[Rakhlin et~al.(2012)Rakhlin, Shamir, and Sridharan]{RakhlinSS12}
A.~Rakhlin, O.~Shamir, and K.~Sridharan.
\newblock Making gradient descent optimal for strongly convex stochastic
  optimization.
\newblock In \emph{ICML}, 2012.

\bibitem[Reddi et~al.(2017)Reddi, Zaheer, Sra, Poczos, Bach, Salakhutdinov, and
  Smola]{Reddi2017escaping}
S.~Reddi, M.~Zaheer, S.~Sra, B.~Poczos, F.~Bach, R.~Salakhutdinov, and
  A.~Smola.
\newblock A generic approach for escaping saddle points.
\newblock \emph{arXiv preprint arXiv:1709.01434}, 2017.

\bibitem[Robbins and Monro(1951)]{RobbinsM51}
H.~Robbins and S.~Monro.
\newblock A stochastic approximation method.
\newblock \emph{The Annals of Mathematical Statistics}, vol. 22, 1951.

\bibitem[Roux et~al.(2012)Roux, Schmidt, and Bach]{RouxSB12}
N.~L. Roux, M.~Schmidt, and F.~R. Bach.
\newblock A stochastic gradient method with an exponential convergence rate for
  strongly-convex optimization with finite training sets.
\newblock In \emph{NIPS 25}, 2012.

\bibitem[Roy and Shynk(1990)]{RoyS90}
S.~Roy and J.~J. Shynk.
\newblock Analysis of the momentum lms algorithm.
\newblock \emph{IEEE Transactions on Acoustics, Speech and Signal Processing},
  1990.

\bibitem[Shalev-Shwartz and Zhang(2012)]{ShwartzZ12}
S.~Shalev-Shwartz and T.~Zhang.
\newblock Stochastic dual coordinate ascent methods for regularized loss
  minimization.
\newblock \emph{CoRR}, abs/1209.1873, 2012.

\bibitem[Shalev-Shwartz and Zhang(2014)]{ShwartzZ14}
S.~Shalev-Shwartz and T.~Zhang.
\newblock Accelerated proximal stochastic dual coordinate ascent for
  regularized loss minimization.
\newblock In \emph{ICML}, 2014.

\bibitem[Sharma et~al.(1998)Sharma, Sethares, and Bucklew]{SharmaSB98}
R.~Sharma, W.~A. Sethares, and J.~A. Bucklew.
\newblock Analysis of momentum adaptive filtering algorithms.
\newblock \emph{IEEE Transactions on Signal Processing}, 1998.

\bibitem[Sutskever et~al.(2013)Sutskever, Martens, Dahl, and
  Hinton]{sutskever2013importance}
I.~Sutskever, J.~Martens, G.~Dahl, and G.~Hinton.
\newblock On the importance of initialization and momentum in deep learning.
\newblock In \emph{International conference on machine learning}, pages
  1139--1147, 2013.

\bibitem[Tieleman and Hinton(2012)]{Hinton2012RMS}
T.~Tieleman and G.~Hinton.
\newblock Lecture 6.5-rmsprop: Divide the gradient by a running average of its
  recent magnitude.
\newblock \emph{COURSERA: Neural networks for machine learning}, 2012.

\bibitem[Zhang et~al.(2017)Zhang, Mitliagkas, and Ré]{Zhang2017yellowfin}
J.~Zhang, I.~Mitliagkas, and C.~Ré.
\newblock Yellowfin and the art of momentum tuning.
\newblock \emph{CoRR}, abs/1706.03471, 2017.

\end{thebibliography}

\pagebreak
\appendix
\section{Suboptimality of~\hb: Proof of Proposition~\ref{prop:hb}}\label{app:hb-proof}
Before proceeding to the proof, we introduce some additional notation. Let $\thetav_{t+1}^{(j)}$ denote the concatenated and centered estimates in the $j^{\textrm{th}}$ direction for $j=1,2$.
\begin{align*}
\thetav_{t+1}^{(j)} \eqdef \begin{bmatrix}\w_{t+1}^{(j)}-(\w^*)^{(j)}\\\w_{t}^{(j)}-(\w^*)^{(j)}\end{bmatrix}, \quad j = 1,2.
\end{align*}
Since the distribution over $x$ is such that the coordinates are decoupled, we see that $\thetav_{t+1}^{(j)}$ can be written in terms of $\thetav_{t}^{(j)}$ as:
\begin{align*}
\thetav_{t+1}^{(j)} = \Ah_{t+1}^{(j)}\thetav_t^{(j)}, \mbox{ with } \Ah_{t+1}^{(j)}=\begin{bmatrix}1+\alpha-\delta (a_{t+1}^{(j)})^2&-\alpha\\1&0\end{bmatrix}.
\end{align*}
Let $\phiv_{t+1}^{(j)} \eqdef \E{\thetav_{t+1}^{(j)}\otimes\thetav_{t+1}^{(j)}}$ denote the covariance matrix of $\thetav_{t+1}^{(j)}$.
We have $\phiv_{t+1}^{(j)} = \BT^{(j)} \phiv_{t}^{(j)}$ with, $\BT^{(j)}$ defined as
\begin{align*}
\BT^{(j)} &\eqdef \begin{bmatrix} \E{(1+\alpha-\delta(a^{(j)})^2)^2} & \E{-\alpha(1+\alpha-\delta(a^{(j)})^2)}& \E{-\alpha(1+\alpha-\delta(a^{(j)})^2} & \alpha^2\\ \E{(1+\alpha-\delta(a^{(j)})^2)} &0&-\alpha&0\\\E{(1+\alpha-\delta(a^{(j)})^2)} & -\alpha & 0 & 0\\1 & 0 & 0 &0 \end{bmatrix}\\
&=\begin{bmatrix}(1+\alpha-\delta\sigma_j^2)^2+(c-1)(\delta\sigma_j^2)^2&-\alpha(1+\alpha-\delta\sigma_j^2)&-\alpha(1+\alpha-\delta\sigma_j^2)&\alpha^2\\
(1+\alpha-\delta\sigma_j^2)&0&-\alpha&0\\
(1+\alpha-\delta\sigma_j^2)&-\alpha&0&0\\
1&0&0&0\end{bmatrix}.
\end{align*}
We prove Proposition~\ref{prop:hb} by showing that for any choice of stepsize and momentum, either of the two holds: \begin{itemize}
	\item $\BT^{(1)}$ has an eigenvalue larger than $1$, or, 
	\item the largest eigenvalue of $\BT^{(2)}$ is greater than $1 - \frac{500}{\kappa}$.
\end{itemize}
This is formalized in the following two lemmas.
\begin{lemma}\label{lem:top-eig}
	If the stepsize $\delta$ is such that $\delta \sigma_1^2 \geq \frac{2\left(1-\alpha^2\right)}{c+(c-2)\alpha}$, then $\BT^{(1)}$ has an eigenvalue $\geq 1$.
\end{lemma}
\begin{lemma}\label{lem:bottom-eig}
	If the stepsize $\delta$ is such that $\delta \sigma_1^2 < \frac{2\left(1-\alpha^2\right)}{c+(c-2)\alpha}$, then $\BT^{(2)}$ has an eigenvalue of magnitude $\geq 1-\frac{500}{\kappa}$.
\end{lemma}

Given this notation, we can now consider the $j^{th}$ dimension without the superscripts; when needed, they will be made clear in the exposition. Denoting $x\eqdef\delta\sigma^2$ and $t\eqdef1+\alpha-x$, we have:
\begin{align*}
\BT = \begin{bmatrix} t^2 + (c-1) x^2 & -\alpha t & -\alpha t & \alpha^2 \\
t & 0 & -\alpha & 0\\
t & -\alpha & 0 & 0\\
1 & 0 & 0 & 0\end{bmatrix}
\end{align*}

\subsection{Proof}
The analysis goes via computation of the characteristic polynomial of $\BT$ and evaluating it at different values to obtain bounds on its roots.
\begin{lemma}\label{lem:char}
	The characteristic polynomial of $\BT$ is:
	\begin{align*}
	D(z)=z^4-(t^2+(c-1)x^2)z^3+(2\alpha t^2-2\alpha^2)z^2+(-t^2+(c-1)x^2)\alpha^2 z+\alpha^4.
	\end{align*}
\end{lemma}
\begin{proof}
	We first begin by writing out the expression for the determinant:
	\begin{align*}
	Det(\BT-z\eyeT)=\begin{vmatrix} t^2+(c-1)x^2-z & -\alpha t & -\alpha t & \alpha^2\\ 
	t & -z & -\alpha & 0\\
	t & -\alpha & -z & 0\\
	1 & 0 & 0 & -z \end{vmatrix}.
	\end{align*}
	expanding along the first column, we have:{\small
	\begin{align*}
	Det(\BT-z\eyeT)&=(t^2+(c-1)x^2-z)(\alpha^2z-z^3)-t(-\alpha t z^2 + \alpha^2 t z)+t(-\alpha t (\alpha z)+z\cdot \alpha t z)-(z\cdot \alpha^2 z - \alpha^4)\\
	&=(t^2+(c-1)x^2-z)(\alpha^2z-z^3)-2t(\alpha^2 tz-\alpha t z^2)-(\alpha^2 z^2-\alpha^4).
	\end{align*}}
	Expanding the terms yields the expression in the lemma.
\end{proof}

The next corollary follows by some simple arithmetic manipulations.
\begin{corollary}
	\label{cor:ceq}
	Substituting $z=1-\tau$ in the characteristic equation of Lemma~\ref{lem:char}, we have:
	\begin{align}
	\label{eq:charEqnMod}
	D(1-\tau) &= \tau^4 + \tau^3 (-4 + t^2 + (c-1) x^2) + \tau^2 (6-3t^2-3(c-1)x^2-2\alpha^2+2\alpha t^2)\nonumber\\ 
	&+\tau(-4+3t^2+3(c-1)x^2+4\alpha^2-4\alpha t^2-(c-1)x^2\alpha^2+t^2\alpha^2)\nonumber\\
	&+ (1-t^2-(c-1)x^2-2\alpha^2+2\alpha t^2 +(c-1)x^2\alpha^2-t^2\alpha^2+\alpha^4)\nonumber\\
	&=\tau^4 + \tau^3[-(3+\alpha)(1-\alpha)-2x(1+\alpha)+cx^2]\nonumber\\
	&+\tau^2[(3-4\alpha-\alpha^2+2\alpha^3)-2x(1+\alpha)(2\alpha-3)+x^2(2\alpha-3c)]\nonumber\\
	&+\tau[-(1-\alpha)^2(1-\alpha^2)-2x(3-\alpha)(1-\alpha^2)+x^2(3c-4\alpha+(2-c)\alpha^2)]\nonumber\\
	&+x(1-\alpha)[2(1-\alpha^2)-x(c+(c-2)\alpha)].
	\end{align}
\end{corollary}

%
\begin{proof}[Proof of Lemma~\ref{lem:top-eig}]
	The first observation necessary to prove the lemma is that the characteristic polynomial $D(z)$ approaches $\infty$ as $z\to\infty$, i.e., $\lim_{z\to\infty}D(z)=+\infty$. 
	
	Next, we evaluate the characteristic polynomial at $1$, i.e. compute $D(1)$. This follows in a straightforward manner from corollary~\eqref{cor:ceq} by substituting $\tau=0$ in equation~\eqref{eq:charEqnMod}, and this yields,
	\begin{align*}
	D(1)=(1-\alpha) x \cdot \bigg(2(1-\alpha^2)-x(1-\alpha)-(c-1)x(1+\alpha)\bigg).
	\end{align*}
	As $\alpha<1$, $x=\delta\sigma^2>0$, we have the following by setting $D(1)\leq0$ and solving for $x$:
	\begin{align*}
	x\geq\frac{2(1-\alpha^2)}{c+(c-2)\alpha}.
	\end{align*}
	Since $D(1)\leq 0$ and $D(z)\geq 0$ as $z\rightarrow \infty$, there exists a root of $D(\cdot)$ which is $\geq 1$.
\end{proof}
\begin{remark}
	The above characterization is striking in the sense that for any $c>1$, increasing the momentum parameter $\alpha$ naturally requires the reduction in the step size $\delta$ to permit the convergence of the algorithm, which is not observed when fast gradient methods are employed in deterministic optimization. For instance, in the case of deterministic optimization, setting $c=1$ yields $\delta\sigma_1^2<2(1+\alpha)$. On the other hand, when employing the stochastic heavy ball method with $x^{(j)} = 2 \sigma_j^2$, we have the condition that $c = 2$, and this implies, $\delta \sigma_1^2<\frac{2(1-\alpha^2)}{2}=1-\alpha^2$.
\end{remark}

We now prove Lemma~\ref{lem:bottom-eig}. We first consider the large momentum setting.
\begin{lemma}\label{lem:large-mom}
	When the momentum parameter $\alpha$ is set such that $1-450/\cnH\leq\alpha\leq 1$, $\BT$ has an eigenvalue of magnitude $\geq 1-\frac{450}{\cnH}$.
\end{lemma}
\begin{proof}
	This follows easily from the fact that $\text{det}(\BT)=\alpha^4=\prod_{j=1}^4\lambda_j(\BT)\leq(\lambda_{\max}(\BT))^4$, thus implying $1-450/\cnH \leq \alpha \leq \abs{\lambda_{\max}(\BT)}$.
\end{proof}
\begin{remark}
	Note that the above lemma holds for any value of the learning rate $\delta$, and holds for every eigen direction of $\H$. Thus, for ``large'' values of momentum, the behavior of stochastic heavy ball does degenerate to the behavior of stochastic gradient descent.
\end{remark}

We now consider the setting where momentum is bounded away from $1$.
\begin{corollary}
Consider $\BT^{(2)}$, by substituting $\tau=\l/\cnH$, $x=\delta\lambda_{\min}=c(\delta\sigma_1^2)/\cnH$ in equation~\eqref{eq:charEqnMod} and accumulating terms in varying powers of $1/\cnH$, we obtain:
\begin{align}
\label{eq:charEqnMod2_1}
G(\l)&\eqdef\frac{c^3(\delta\sigma_1^2)^2\l^3}{\cnH^5}+\frac{\l^4-2c(\delta\sigma_1^2)\l^3(1+\alpha)+(2\alpha-3c)c^2(\delta\sigma_1^2)^2\l^2}{\cnH^4}\nonumber\\&+\frac{-(3+\alpha)(1-\alpha)\l^3-2(1+\alpha)(2\alpha-3)c(\delta\sigma_1^2)\l^2+(3c-4\alpha+(2-c)\alpha^2)c^2(\delta\sigma_1^2)^2\l}{\cnH^3}\nonumber\\&+\frac{(3-4\alpha-\alpha^2+2\alpha^3)\l^2-2c(\delta\sigma_1^2)\l(3-\alpha)(1-\alpha^2)-c^2(\delta\sigma_1^2)^2(1-\alpha)(c+(c-2)\alpha)}{\cnH^2}\nonumber\\&+\frac{-(1-\alpha)^2(1-\alpha^2)\l+2c(\delta\sigma_1^2)(1-\alpha)(1-\alpha^2)}{\cnH}
\end{align}
\end{corollary}

\begin{lemma}\label{lem:small-mom}
Let $2<c<3000$, $0\leq\alpha\leq1-\frac{450}{\cnH}$, $\l=1+\frac{2c(\delta\sigma_1^2)}{1-\alpha}$. Then, $G(\l)\leq0$.
\end{lemma}
\begin{proof}
Since $(\delta\sigma_1^2) \leq \frac{2(1-\alpha^2)}{c+(c-2)\alpha}$, this implies $\frac{(\delta\sigma_1^2)}{1-\alpha}\leq\frac{2(1+\alpha)}{c+(c-2)\alpha}\leq\frac{4}{c}$, thus implying, $1\leq \l\leq 9$.

Substituting the value of $\l$ in equation~\eqref{eq:charEqnMod2_1}, the coefficient of $\mathcal{O}(1/\cnH)$ is $-(1-\alpha)^3(1+\alpha)$. 

We will bound this term along with $(3-4\alpha-\alpha^2+2\alpha^3)\l^2/\cnH^2=(1-\alpha)^2(3+2\alpha)\l^2/\cnH^2$ to obtain:
\begin{align*}
\frac{-(1-\alpha)^3(1+\alpha)}{\cnH}+\frac{(1-\alpha)^2(3+2\alpha)\l^2}{\cnH^2}&\leq\frac{-(1-\alpha)^3(1+\alpha)}{\cnH}+\frac{405(1-\alpha)^2}{\cnH^2}\\
&\leq\frac{(1-\alpha)^2}{\cnH}\bigg(\frac{405}{\cnH}-(1-\alpha^2)\bigg)\\
&\leq\frac{(1-\alpha)^2}{\cnH}\bigg(\frac{405}{\cnH}-(1-\alpha)\bigg)\leq-\frac{45\cdot 450^2}{\cnH^4},
\end{align*}
where, we use the fact that $\alpha<1$, $\l\leq9$.
The natural implication of this bound is that the terms that are lower order, such as $\mathcal{O}(1/\cnH^4)$ and $\mathcal{O}(1/\cnH^5)$ will be negative owing to the large constant above. Let us verify that this is indeed the case by considering the terms having powers of $\mathcal{O}(1/\cnH^4)$ and $\mathcal{O}(1/\cnH^5)$ from equation~\eqref{eq:charEqnMod2_1}:
\begin{align*}
&\frac{c^3(\delta\sigma_1^2)^2\l^3}{\cnH^5}+\frac{\l^4-2c(\delta\sigma_1^2)\l^3(1+\alpha)+(2\alpha-3c)c^2(\delta\sigma_1^2)^2\l^2}{\cnH^4}-\frac{45\cdot 450^2}{\cnH^4}\nonumber\\
&\leq \frac{c^3(\delta\sigma_1^2)^2\l^3}{\cnH^5} + \frac{l^4}{\cnH^4}-\frac{45\cdot 450^2}{\cnH^4}\nonumber\\
&\leq \frac{c \l^3}{\cnH^5}+\frac{(9^4-(45\cdot450^2))}{\cnH^4}\leq \frac{9^3 c + 9^4-(45\cdot450^2)}{\cnH^4}
\end{align*}
The expression above evaluates to $\leq0$ given an upperbound on the value of $c$. The expression above follows from the fact that $\l\leq9,\cnH\geq 1$.

Next, consider the terms involving $\mathcal{O}(1/\cnH^3)$ and $\mathcal{O}(1/\cnH^2)$, in particular,
\begin{align*}
&\frac{(3c-4\alpha+(2-c)\alpha^2)c^2(\delta\sigma_1^2)^2\l}{\cnH^3}-\frac{c^2(\delta\sigma_1^2)^2(1-\alpha)(c+(c-2)\alpha)}{\cnH^2}\\
&\leq \frac{c^2(\delta\sigma_1^2)^2}{\cnH^2}\big( \frac{\l(3c+2)}{\cnH} - (1-\alpha)(c+(c-2)\alpha)\big)\\
&\leq \frac{c^2(\delta\sigma_1^2)^2}{\cnH^2}\big( \frac{5c\l}{\cnH} - (1-\alpha)(c+(c-2)\alpha)\big)\\
&\leq \frac{c^2(\delta\sigma_1^2)^2}{\cnH^2}\big( \frac{5c\l}{\cnH} - (1-\alpha)c\big)\\
&\leq \frac{c^3(\delta\sigma_1^2)^2}{\cnH^2}\big(\frac{5\l}{\cnH}-\frac{450}{\cnH}\big)\\
&\leq \frac{c^3(\delta\sigma_1^2)^2}{\cnH^2}\cdot \frac{-405}{\cnH}\leq 0.
\end{align*}

Next,
\begin{align*}
&\frac{-2(1+\alpha)(2\alpha-3)c(\delta\sigma_1^2)\l^2}{\cnH^3}-\frac{2c(\delta\sigma_1^2)\l(3-\alpha)(1-\alpha^2)}{\cnH^2}\\
&\leq \frac{2(1+\alpha)c(\delta\sigma_1^2) \l}{\cnH^2}\bigg( \frac{-(2\alpha-3)\l}{\cnH} - (3-\alpha)(1-\alpha)\bigg)\\
&\leq \frac{2(1+\alpha)c(\delta\sigma_1^2) \l}{\cnH^2}\bigg( \frac{3\l}{\cnH} - 2(1-\alpha)\bigg)\\
&\leq \frac{2(1+\alpha)c(\delta\sigma_1^2) \l}{\cnH^2}\bigg( \frac{3\l}{\cnH} - \frac{2\cdot450}{\cnH}\bigg)\\
&\leq \frac{2(1+\alpha)c(\delta\sigma_1^2) \l}{\cnH^2}\bigg( \frac{3\cdot 27}{\cnH} - \frac{2\cdot450}{\cnH}\bigg)\leq 0.
\end{align*}
In both these cases, we used the fact that $\alpha\leq1-\frac{450}{\cnH}$ implying $-(1-\alpha)\leq\frac{-450}{\cnH}$.
Finally, other remaining terms are negative.
\end{proof}

Before rounding up the proof of the proposition, we need the following lemma to ensure that our lower bounds on the largest eigenvalue of $\BT$ indeed affect the algorithm's rates and are true \emph{irrespective} of where the algorithm is begun. Note that this allows our result to be \emph{much stronger} than typical optimization lowerbounds that rely on specific initializations to ensure a component along the largest eigendirection of the update operator, for which bounds are proven.
\begin{lemma}
\label{lem:divergentLemma}
For any starting iterate $\w_0\ne\w^*$, the \hb method produces a non-zero component along the largest eigen direction of $\BT$.
\end{lemma}
\begin{proof}
We note that in a similar manner as other proofs, it suffices to argue for each dimension of the problem separately. But before we start looking at each dimension separately, let us consider the $j^{\text{th}}$ dimension, and detail the approach we use to prove the claim: the idea is to examine the subspace spanned by covariance $\E{\thetav^{(j)}_\cdot\otimes\thetav^{(j)}_\cdot}$ of the iterates $\thetav^{(j)}_0,\thetav^{(j)}_1,\thetav^{(j)}_2,...$, for every starting iterate $\thetav^{(j)}_0\ne \begin{bmatrix}0,0\end{bmatrix}^{\top}$ and prove that the largest eigenvector of the expected operator $\BT^{(j)}$ is not orthogonal to this subspace. This implies that there exists a non-zero component of $\E{\thetav^{(j)}_\cdot\otimes\thetav^{(j)}_\cdot}$ in the largest eigen direction of $\BT^{(j)}$, and this decays at a rate that is at best $\lambda_{\max}(\BT^{(j)})$.

Since $\BT^{(j)}\in\R^{4\times 4}$, we begin by examining the expected covariance spanned by the iterates $\thetav^{(j)}_0,\thetav^{(j)}_1,\thetav^{(j)}_2,\thetav^{(j)}_3$. Let $\w^{(j)}_0-(\w^*)^{(j)}=\w^{(j)}_{-1}-(\w^*)^{(j)}=\init^{(j)}$. Now, this implies $\thetav^{(j)}_0=\init^{(j)}\cdot\begin{bmatrix} 1,1\end{bmatrix}^{\top}$. Then,
\begin{align*}
\thetav^{(j)}_1=k^{(j)}\Ah^{(j)}_{1}\begin{bmatrix}1\\1\end{bmatrix}, \text{with  } \Ah^{(j)}_{1} = \begin{bmatrix}1+\alpha-\delta \Hhat^{(j)}_1 & -\alpha \\ 1 & 0\end{bmatrix}, \text{ where  } \Hhat^{(j)}_1=(a^{(j)}_1)^2.
\end{align*}
This implies that $\init$ just appears as a scale factor. This in turn implies that in order to analyze the subspace spanned by the covariance of iterates $\thetav^{(j)}_0,\thetav^{(j)}_1,...$, we can assume $\init^{(j)}=1$ without any loss in generality. This implies, $\thetav^{(j)}_0 = \begin{bmatrix} 1, 1 \end{bmatrix}^{\top}$. Note that with this in place, we see that we can now drop the superscript $j$ that represents the dimension, since the analysis decouples across the dimensions $j\in\{1,2\}$. Furthermore, let the entries of the vector $\thetav_k$ be represented as $\thetav_k\eqdef\begin{bmatrix}\theta_{k1}&\theta_{k2}\end{bmatrix}^{\top}$ Next, denote $1+\alpha-\delta\Hhat_k = \th_k$. This implies,
\begin{align*}
\Ah_k  = \begin{bmatrix}  \th_k  & -\alpha \\1 & 0\end{bmatrix}.
\end{align*}
Furthermore, 
\begin{align}
\label{eq:thetas}
\thetav_1 = \Ah_1\thetav_0 = \begin{bmatrix}\th_1-\alpha\\1\end{bmatrix},\ \thetav_2 = \Ah_2\thetav_1 = \begin{bmatrix} \th_2(\th_1-\alpha)-\alpha \\ \th_1-\alpha\end{bmatrix},\nonumber\\ \ \thetav_3 = \Ah_3\thetav_2 = \begin{bmatrix} \th_3(\th_2(\th_1-\alpha)-\alpha)-\alpha(\th_1-\alpha)\\\th_2(\th_1-\alpha)-\alpha\end{bmatrix}.
\end{align}
Let us consider the vectorized form of $\phiv_j=\E{\thetav_j\otimes\thetav_j}$, and we denote this as $\text{vec}(\phiv_j)$. Note that $\text{vec}(\phiv_j)$ makes $\phiv_j$ become a column vector of size $4\times 1$. Now, consider $\text{vec}(\phiv_j)$ for $j=0,1,2,3$ and concatenate these to form a matrix that we denote as  $\DT$, i.e.
\begin{align*}
\DT &= \begin{bmatrix} \text{vec}(\phiv_0) & \text{vec}(\phiv_1) & \text{vec}(\phiv_2) & \text{vec}(\phiv_3) \end{bmatrix}.
\end{align*}
Now, since we note that $\phiv_j$ is a symmetric $2\times 2$ matrix, $\DT$ should contain two identical rows implying that it has an eigenvalue that is zero and a corresponding eigenvector that is $\begin{bmatrix} 0 &-1/\sqrt{2} & 1/\sqrt(2) & 0\end{bmatrix}^{\top}$. It turns out that this is also an eigenvector of $\BT$ with an eigenvalue $\alpha$. Note that $\text{det}(\BT)=\alpha^4$. This implies there are two cases that we need to consider: (i) when all eigenvalues of $\BT$ have the same magnitude ($=\alpha$). In this case, we are already done, because there exists at least one non zero eigenvalue of $\DT$ and this should have some component along one of the eigenvectors of $\BT$ and we know that all eigenvectors have eigenvalues with a magnitude equal to $\lambda_{\max}(\BT)$. Thus, there exists an iterate which has a non-zero component along the largest eigendirection of $\BT$. (ii) the second case is the situation when we have eigenvalues with different magnitudes. In this case, note that $\text{det}(\BT)=\alpha^4<(\lambda_{\max}(\BT))^4$ implying $\lambda_{\max}(\BT)>\alpha$. In this case, we need to prove that $\DT$ spans a three-dimensional subspace; if it does, it contains a component along the largest eigendirection of $\BT$ which will round up the proof. Since we need to understand whether $\DT$ spans a three dimensional subspace, we can consider a different (yet related) matrix, which we call $\RT$ and this is defined as:
\begin{align*}
\RT\eqdef \mathbb{E}\biggl(\begin{bmatrix} \theta_{01}^2 & \theta_{11}^2 & \theta_{21}^2 \\ \theta_{01}\theta_{02} & \theta_{11}\theta_{12} & \theta_{21}\theta_{22} \\\theta_{02}^2 & \theta_{12}^2 & \theta_{22}^2\end{bmatrix}\biggr)
\end{align*}
Given the expressions for $\{\thetav_j\}_{j=0}^3$ (by definition of $\thetav_0$ and using equation~\ref{eq:thetas}), we can substitute to see that $\RT$ has the following expression:
\begin{align*}
\RT = \begin{bmatrix} 1 & \E{(\th_1-\alpha)^2} & \E{(\th_2(\th_1-\alpha)-\alpha)^2} \\
1 & \E{\th_1-\alpha} & \E{((\th_2(\th_1-\alpha)-\alpha))(\th_1-\alpha)}\\
1 & 1 & \E{(\th_1-\alpha)^2}\end{bmatrix}.
\end{align*}
If we compute and prove that $\text{det}(\RT)\ne 0$, we are done since that implies that $\RT$ has three non-zero eigenvalues.

This implies, we first define the following: let $q_\gamma = (t-\gamma)^2 + (c-1)x^2$. Then, $\RT$ can be expressed as:
\begin{align*}
\text{det}(\RT) &= \text{det}\bigg(\begin{bmatrix} 1 & q_{\alpha} & q_0q_\alpha-2\alpha t (t-\alpha) + \alpha^2 \\ 1 & t-\alpha & tq_{\alpha}-\alpha(t-\alpha) \\ 1 & 1 & q_{\alpha}  \end{bmatrix}\bigg)\\
&=\text{det}\bigg(\begin{bmatrix} 1 & q_\alpha & q_\alpha(q_0-q_\alpha)-2\alpha t(t-\alpha) + \alpha^2 \\
1 & t-\alpha & tq_{\alpha}-\alpha(t-\alpha)-(t-\alpha)q_\alpha\\
1 & 1 & 0\end{bmatrix}\bigg)\\
&=\text{det}\bigg(\begin{bmatrix} 1 & q_\alpha-1 & q_\alpha(q_0-q_\alpha)-2\alpha t(t-\alpha) + \alpha^2 \\
1 & t-\alpha-1 & tq_{\alpha}-\alpha(t-\alpha)-(t-\alpha)q_\alpha\\
1 & 0 & 0\end{bmatrix}\bigg)\\
&=\text{det}\bigg(\begin{bmatrix} 0 & q_\alpha-1 & q_\alpha(q_0-q_\alpha)-2\alpha t(t-\alpha) + \alpha^2 \\
0 & t-\alpha-1 & tq_{\alpha}-\alpha(t-\alpha)-(t-\alpha)q_\alpha\\
1 & 0 & 0\end{bmatrix}\bigg)
\end{align*}
Note:
(i) $q_\alpha-1 = (t-\alpha)^2-1+(c-1)x^2 = (1-x)^2-1+(c-1)x^2=-2x+x^2+(c-1)x^2=-2x+cx^2$.\\
(ii) $t-\alpha-1=-x$\\
(iii) $\alpha(q_\alpha-(t-\alpha))=\alpha((t-\alpha)^2-(t-\alpha)+(c-1)x^2)=\alpha((1-x)(-x)+(c-1)x^2)=\alpha x(-1+cx)$\\
(iv) $q_0-q_\alpha=t^2-(t-\alpha)^2=\alpha(2t-\alpha)=2t\alpha-\alpha^2$.\\
Then,
\begin{align*}
(2\alpha t-\alpha^2)q_\alpha-2\alpha t(t-\alpha)+\alpha^2&=2t\alpha(q_{\alpha}-(t-\alpha)) + \alpha^2 (1-q_{\alpha})\\
&=2t\alpha(-x+cx^2) - \alpha^2 (-2x + cx^2)\\
&=-2t\alpha x + 2x\alpha^2 + 2t\alpha c x^2 - c\alpha^2 x^2\\
&=2\alpha x(-t+\alpha) + c\alpha x^2 (2t-\alpha)\\
&=-2\alpha x(1-x) + 2c\alpha x^2(1-x) + c\alpha^2x^2\\
&=2\alpha x(1-x)(-1+cx) + c\alpha^2 x^2.
\end{align*}
Then,
\begin{align*}
\text{det}(\RT) &= \text{det}\bigg(\begin{bmatrix} 0 & x(cx-2) & 2\alpha x(1-x)(-1+cx) + c\alpha^2 x^2\\
0 & -x & \alpha x(cx-1)\\
1 & 0 & 0\end{bmatrix}\bigg)\\
&=x^2\alpha\text{det}\bigg(\begin{bmatrix} 0 & (cx-2) & c\alpha x+2(1-x)(cx-1)\\
0 & -1 & cx-1\\
1 & 0 & 0\end{bmatrix}\bigg)\\
&=x^3\alpha\text{det}\bigg(\begin{bmatrix} 0 & c & c\alpha -2(cx-1)\\
0 & -1 & cx-1\\
1 & 0 & 0\end{bmatrix}\bigg)
\end{align*}
Then,
\begin{align*}
\text{det}(\RT)&=x^3\alpha\bigg( c(-1+cx)-2(-1+cx)+c\alpha  \bigg)\\
&= \alpha x^3 \bigg( (c-2)(-1+cx) + c\alpha  \bigg)
\end{align*}
Note that this determinant can be zero when 
\begin{align}
\label{eq:detZeroCondition}
\alpha = \frac{(c-2)(1-cx)}{c}.
\end{align}
We show this is not possible by splitting our argument into two parts, one about the convergent regime of the algorithm (where, $\delta\sigma_1^2<\frac{2(1-\alpha^2)}{c+(c-2)\alpha}$) and the other about the divergent regime. 

Let us first provide a proof for the convergent regime of the algorithm. For this regime, let the chosen $\delta$ be represented as $\delta^+$. Now, for the smaller eigen direction, $x = \delta^+ \lambda_{\min} = c\delta^+\sigma_1^2/\cnH$. Suppose $\alpha$ was chosen as per equation~\ref{eq:detZeroCondition}, 
\begin{align*}
\frac{c\alpha}{c-2} &= 1 - \frac{c^2\delta^+\sigma_1^2}{\cnH}\\
\implies \delta^+\sigma_1^2 &= \frac{\cnH}{c^2} - \frac{\cnH\alpha}{c(c-2)}.
\end{align*}
We will now prove that $\delta^+\sigma_1^2=\frac{\cnH}{c}(\frac{1}{c}-\frac{\alpha}{c-2})$ is much larger than one allowed by the convergence of the HB updates, i.e., $\delta\sigma_1^2<\frac{2(1-\alpha^2)}{c+(c-2)\alpha}\leq\frac{2(1-\alpha^2)}{c}$. In particular, if we prove that $\frac{\cnH}{c}(\frac{1}{c}-\frac{\alpha}{c-2})>\frac{2(1-\alpha^2)}{c}$ for any admissible value of $\alpha$, we are done. \begin{align*}
\frac{\cnH}{c}(\frac{1}{c}-\frac{\alpha}{c-2})&>\frac{2(1-\alpha^2)}{c}\\
\Leftrightarrow  \frac{\cnH}{c} - \frac{\cnH\alpha}{c-2}&> 2-2\alpha^2\\
\Leftrightarrow \frac{\cnH}{c} - \frac{\cnH\alpha}{c-2} >\frac{\cnH}{c}-\frac{\cnH\alpha}{c}&>2-2\alpha^2\\
\Leftrightarrow \cnH-\cnH\alpha>2c-2c\alpha^2\\
\Leftrightarrow 2c\alpha^2 - \cnH\alpha + (\cnH-2c)> 0.
\end{align*}
The two roots of this quadratic equation are $\alpha^+ = \frac{\cnH}{2c}-1$ and $\alpha^- = 1$. Note that $\cnH\geq\cnS=c$; note that there is not much any method gains over SGD if $\cnH=\mathcal{O}(c)$. And, for any $\cnH\geq4c$, note, $\alpha^+>\alpha^-$, indicating that the above equation holds true if $\alpha>\alpha^+ = \frac{\cnH}{2c}-1$ or if $\alpha < \alpha^- = 1$. The latter condition is true and hence the proposition that $\delta^+\sigma_1^2>\frac{2(1-\alpha^2)}{c+(c-2)\alpha}$ is true. 

We need to prove that the determinant does not vanish in the divergent regime for rounding up the proof to the lemma.

Now, let us consider the divergent regime of the algorithm, i.e., when, $\delta\sigma_1^2 > \frac{2(1-\alpha^2)}{c+(c-2)\alpha}$. Furthermore, for the larger eigendirection, the determinant is zero when $\delta\sigma_1^2 = \frac{1 - \frac{c\alpha}{c-2}}{c} = \frac{1}{c} - \frac{\alpha}{c-2}$ (obtained by substituting $x=\delta\sigma_1^2$ in equation~\ref{eq:detZeroCondition}). If we show that $\frac{2(1-\alpha^2)}{c+(c-2)\alpha}>\frac{1}{c} - \frac{\alpha}{c-2}$ for all admissible values of $c$, we are done. We will explore this in greater detail:
\begin{align*}
\frac{2(1-\alpha^2)}{c+(c-2)\alpha}&>\frac{1}{c} - \frac{\alpha}{c-2}\\
\Leftrightarrow 2(1-\alpha^2)&\geq 1 + \frac{c-2}{c}\alpha - \frac{c}{c-2}\alpha - \alpha^2\\
\Leftrightarrow 1-\alpha^2 &\geq \frac{-4(c-1)}{c(c-2)}\alpha\\
\Leftrightarrow c^2 - 2c -\alpha^2 c^2 +2c\alpha^2 &\geq -4c\alpha + 4\alpha\\
\Leftrightarrow c^2 (1-\alpha^2) -2c (1-\alpha^2-2\alpha) -4\alpha&\geq0.
\end{align*}
considering the quadratic in the left hand size and solving it for $c$, we have:
\begin{align*}
c^{\pm} &= \frac{2(1-\alpha^2-2\alpha) \pm \sqrt{4(1-\alpha^2-2\alpha)^2+16\alpha(1-\alpha^2)}}{2(1-\alpha^2)}\\
&=\frac{(1-\alpha^2-2\alpha) \pm \sqrt{(1-\alpha^2-2\alpha)^2+4\alpha(1-\alpha^2)}}{(1-\alpha^2)}\\
&=\frac{(1-\alpha^2-2\alpha) \pm \sqrt{1+\alpha^4+4\alpha^2-2\alpha^2-4\alpha+4\alpha^3+4\alpha(1-\alpha^2)}}{(1-\alpha^2)}\\
&=\frac{(1-\alpha^2-2\alpha) \pm (1+\alpha^2)}{(1-\alpha^2)}
\end{align*}
This holds true iff
\begin{align*}
c\leq c^- =\frac{-2\alpha(1+\alpha)}{1-\alpha^2}=\frac{-2\alpha}{1-\alpha},
\end{align*}
or iff,
\begin{align*}
c \geq c^+ = \frac{2(1-\alpha)}{1-\alpha^2} = \frac{2}{1+\alpha}.
\end{align*}
Which is true automatically since $c>2$. This completes the proof of the lemma.
\end{proof}

We are now ready to prove Lemma~\ref{lem:bottom-eig}.
\begin{proof}[Proof of Lemma~\ref{lem:bottom-eig}]
	Combining Lemmas~\ref{lem:large-mom} and~\ref{lem:small-mom}, we see that no matter what stepsize and momentum we choose, $\BT^{(j)}$ has an eigenvalue of magnitude at least $1-\frac{500}{\kappa}$ for some $j\in\{1,2\}$. This proves the lemma.
\end{proof}

\section{Equivalence of Algorithm~\ref{alg:asgd} and~\asgd}\label{app:equiv}
We begin by writing out the updates of \asgd as written out in~\citet{jain2017accelerating}, which starts with two iterates $\xa_0$ and $\va_0$, and from time $t=0,1,...T-1$ implements the following updates:
\begin{align}
\ya_t &= \alpp \xa_{t} + (1-\alpp) \va_{t}\label{eq:u1}\\
\xa_{t+1} &= \ya_t -\delp \hat{\nabla}f_{t+1}(\ya_t)\label{eq:u2}\\
\za_{t} &= \betp \ya_{t} + (1-\betp) \va_{t}\label{eq:u3}\\
\va_{t+1} &= \za_t - \gamp \hat{\nabla}f_{t+1}(\ya_t)\label{eq:u4}.
\end{align}
Next, we specify the step sizes $\betp=\cthree^2/\sqrt{\cnH\cnS}$, $\alpp=\cthree/(\cthree+\beta)$, $\gamp=\beta/(\cthree\lambda_{\min})$ and $\delp=1/\infbound$, where $\cnH=\infbound/\lambda_{\min}$. Note that the step sizes in the paper of~\citet{jain2017accelerating} with $\cone$ in their paper set to $1$ yields the step sizes above. Now, substituting equation~\ref{eq:u3} in equation~\ref{eq:u4} and substituting the value of $\gamp$, we have:
\begin{align}
\va_{t+1} &= \betp \left(\ya_{t} - \frac{1}{\cthree\lambda_{\min}}\hat{\nabla}f_{t+1}(\ya_t)\right) + (1-\betp) \va_t\nonumber\\
&= \betp \left(\ya_{t} - \frac{\delta\cnH}{\cthree}\hat{\nabla}f_{t+1}(\ya_t)\right) + (1-\betp) \va_t.\label{eq:u5}
\end{align}
We see that $\va_{t+1}$ is precisely the update of the running average $\bar{w}_{t+1}$ in the \asgd method employed in this paper. 

We now update $\ya_t$ to become $\ya_{t+1}$ and this can be done by writing out equation~\ref{eq:u1} at $t+1$, i.e:
\begin{align}
\ya_{t+1} &= \alpp \xa_{t+1} + (1-\alpp) \va_{t+1} \nonumber\\
&= \alpp \left(\ya_t -\delp \hat{\nabla}f_{t+1}(\ya_t)\right) + (1-\alpp) \va_{t+1}.\label{eq:u6}
\end{align}
By substituting the value of $\alpp$ we note that this is indeed the update of the iterate as a convex combination of the current running average and a short gradient step as written in this paper. In this paper, we set $\cthree$ to be equal to $0.7$, and any constant less than $1$ works. In terms of variables, we note that $\alpha$ in this paper's algorithm description maps to $1-\betp$.
\section{More details on experiments}\label{app:exp-details}
In this section, we will present more details on our experimental setup. 
\subsection{Linear Regression}\label{app:synthetic-det}
In this section, we will present some more results on our experiments on the linear regression problem. Just as in Appendix~\ref{app:hb-proof}, it is indeed possible to compute the expected error of all the algorithms among \sgd, \hb, \nag and \asgd, by tracking certain covariance matrices which evolve as linear systems. For~\sgd, for instance, denoting $\phiv_{t}^{\sgd} \eqdef \E{ \left(\w_{t}^{\sgd} -w^* \right)\otimes \left(\w_{t}^{\sgd} -w^* \right)}$, we see that $\phiv_{t+1}^{\sgd} = \BT \circ \phiv_{t}^{\sgd}$, where $\BT$ is a linear operator acting on $d\times d$ matrices such that $\BT \circ M \eqdef M - \delta HM - \delta MH + \delta^2 \E{\iprod{x}{Mx}xx^\top}$. Similarly, ~\hb,~\nag~and~\asgd~also have corresponding operators (see Appendix~\ref{app:hb-proof} for more details on the operator corresponding to~\hb). The largest magnitude of the eigenvalues of these matrices indicate the rate of decay achieved by the particular algorithm -- smaller it is compared to $1$, faster the decay. 

We now detail the range of parameters explored for these results: the condition number $\kappa$ was varied from $\{2^4,2^5,..,2^{28}\}$ for all the optimization methods and for both the discrete and gaussian problem. For each of these experiments, we draw $1000$ samples and compute the empirical estimate of the fourth moment tensor. For \nag and \hb, we did a very fine grid search by sampling $50$ values in the interval $(0,1]$ for both the learning rate and the momentum parameter and chose the parameter setting that yielded the smallest $\lambda_{\max}(\BT)$ that is less than $1$ (so that it falls in the range of convergence of the algorithm). As for \sgd and \asgd, we employed a learning rate of $1/3$ for the Gaussian case and a step size of $0.9$ for the discrete case. The statistical advantage parameter of \asgd was chosen to be $\sqrt{3\kappa/2}$ for the Gaussian case and $\sqrt{2\kappa/3}$ for the Discrete case, and the a long step parameters of $3\kappa$ and $2\kappa$ were chosen for the Gaussian and Discrete case respectively. The reason it appears as if we choose a parameter above the theoretically maximal allowed value of the advantage parameter is because the definition of $\kappa$ is different in this case. The $\kappa$ we speak about for this experiment is $\lambda_{\max}/\lambda_{\min}$ unlike the condition number for the stochastic optimization problem. In a manner similar to actually running the algorithms (the results of whose are presented in the main paper), we also note that we can compute the rate as in equation~\ref{eq:rate} and join all these rates using a curve and estimate its slope (in the $\log$ scale). This result is indicated in table~\ref{table:synthetic-exp}.

Figure~\ref{fig:synthetic-exp} presents these results, where for each method, we did grid search over all parameters and chose parameters that give smallest $\lambda_{\textrm{max}}$. We see the same pattern as in Figure~\ref{fig:synthetic} from actual runs --~\sgd,\hb~and~\nag~all have linear dependence on condition number $\kappa$, while~\asgd~has a dependence of $\sqrt{\kappa}$.

\begin{figure}[h!]
	\centering
	\subfigure{\includegraphics[width=0.48\columnwidth]{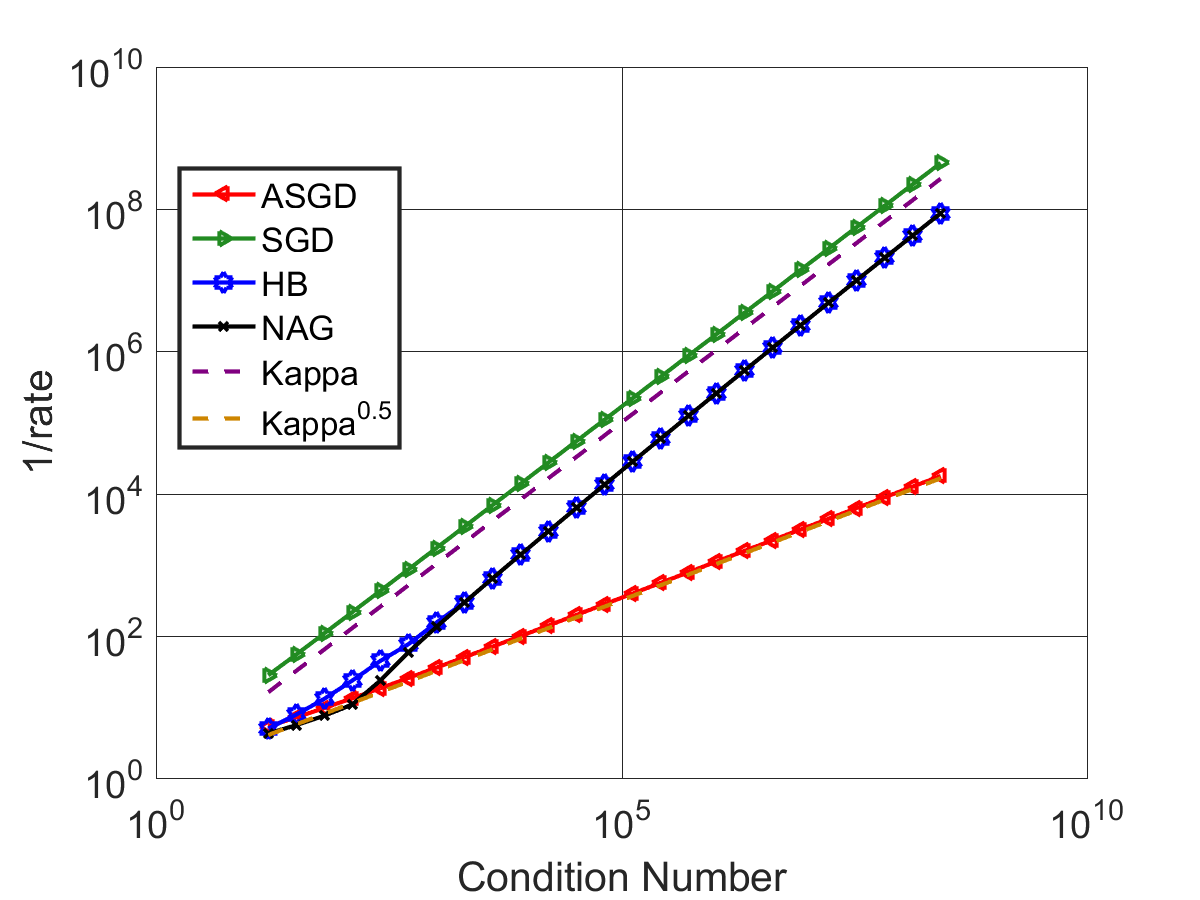}}
	\subfigure{\includegraphics[width=0.48\columnwidth]{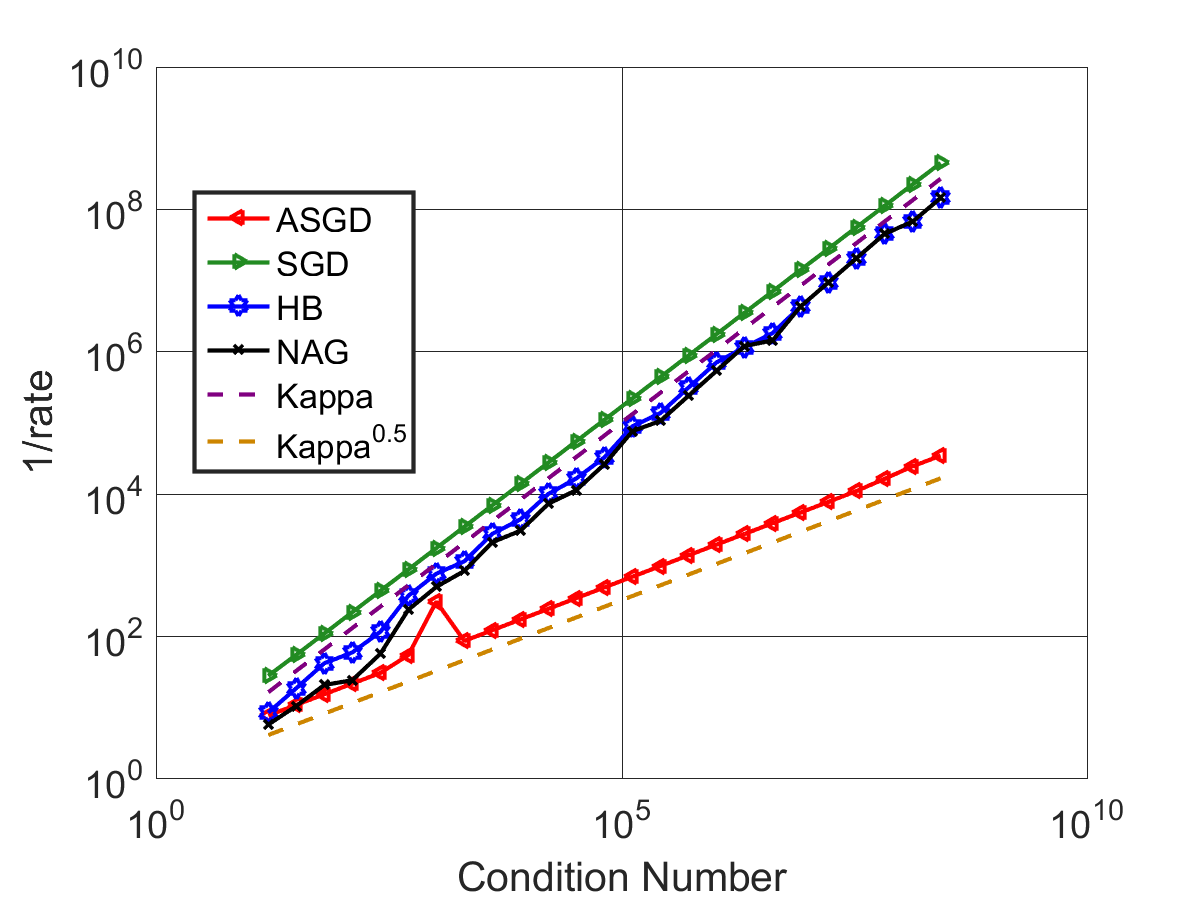}}
	\caption{Expected rate of error decay (equation~\ref{eq:rate}) vs condition number for various methods for the linear regression problem. Left is for discrete distribution and right is for Gaussian distribution.}
	\label{fig:synthetic-exp}
\end{figure}

\begin{table}[h!]
	\centering
	\begin{tabular}{||c | c | c||}
		\hline
		Algorithm & Slope -- discrete & Slope -- Gaussian \\ [0.5ex] 
		\hline\hline
		\sgd & 0.9990 & 0.9995 \\ 
		\hb & 1.0340 & 0.9989 \\
		\nag & 1.0627 & 1.0416 \\
		\asgd & 0.4923 & 0.4906 \\ [1ex] 
		\hline
	\end{tabular}
	\caption{Slopes (i.e. $\gamma$) obtained by fitting a line to the curves in Figure~\ref{fig:synthetic-exp}. A value of $\gamma$ indicates that the error decays at a rate of $\exp\left(\frac{-t}{\kappa^\gamma}\right)$. A smaller value of $\gamma$ indicates a faster rate of error decay.}
	\label{table:synthetic-exp}
\end{table}

\subsection{Autoencoders for MNIST}\label{app:mnist-det}

We begin by noting that the learning rates tend to vary as we vary batch sizes, which is something that is known in theory~\citep{jain2016parallelizing}. Furthermore, we extend the grid especially whenever our best parameters of a baseline method tends to land at the edge of a grid. The parameter ranges explored by our grid search are:

{\bf Batch Size $1$}: (parameters chosen by running for $20$ epochs)
\begin{itemize}
\item \sgd: learning rate: $\{0.01,0.01\sqrt{10},0.1,0.1\sqrt{10},1,\sqrt{10},5,10,20,10\sqrt{10},40,60,80,100$.
\item \nag/\hb: learning rate: $\{0.01\sqrt{10},0.1,0.1\sqrt{10},1,\sqrt{10},10\}$, momentum $\{0,0.5,0.75,0.9,0.95,0.97\}$.
\item \asgd: learning rate: $\{2.5,5\}$, long step $\{100.0,1000.0\}$, advantage parameter $\{2.5,5.0,10.0,20.0\}$.
\end{itemize}

{\bf Batch Size $8$}: (parameters chosen by running for $50$ epochs)
\begin{itemize}
\item \sgd: learning rate: $\{0.001,0.001\sqrt{10.0},0.01,0.01\sqrt{10},0.1,0.1\sqrt{10},1,\sqrt{10},5,10\\,10\sqrt{10},40,60,80,100,120,140\}$.
\item \nag/\hb: learning rate: $\{5.0,10.0,20.0,10\sqrt{10},40,60\}$, momentum $\{0,0.25,0.5,0.75,0.9,0.95\}$.
\item \asgd: learning rate $\{40,60\}$. For a long step of $100$, advantage parameters of $\{1.5,2,2.5,5,10,20\}$. For a long step of $1000$, we swept over advantage parameters of $\{2.5,5,10\}$.
\end{itemize}

\subsection{Deep Residual Networks for CIFAR-10}\label{app:cifar-det}
In this section, we will provide more details on our experiments on~\cifar, as well as present some additional results. We used a weight decay of $0.0005$ in all our experiments. The grid search parameters we used for various algorithms are as follows. Note that the ranges in which parameters such as learning rate need to be searched differ based on batch size~\citep{jain2016parallelizing}. Furthermore, we tend to extrapolate the grid search whenever a parameter (except for the learning rate decay factor) at the edge of the grid has been chosen; this is done so that we always tend to lie in the interior of the grid that we have searched on. Note that for the purposes of the grid search, we choose a hold out set from the training data and add it in to the training data after the parameters are chosen, for the final run. 

{\bf Batch Size $8$}: 
Note: (i) parameters chosen by running for $40$ epochs and picking the grid search parameter that yields the smallest validation $0/1$ error. (ii) The validation set decay scheme that we use is that if the validation error does not decay by at least $1\%$ every three passes over the data, we cut the learning rate by a constant factor (which is grid searched as described below). The minimal learning rate to use is fixed to be $6.25\times 10^{-5}$, so that we do not decay far too many times and curtail progress prematurely.
\begin{itemize}
\item \sgd: learning rate: $\{0.0033,0.01,0.033,0.1,0.33\}$, learning rate decay factor $\{5,10\}$.
\item \nag/\hb: learning rate: $\{0.001,0.0033,0.01,0.033\}$, momentum $\{0.8,0.9,0.95,0.97\}$, learning rate decay factor $\{5,10\}$.
\item \asgd: learning rate $\{0.01,0.0330,0.1\}$, long step $\{1000,10000,50000\}$, advantage parameter $\{5,10\}$, learning rate decay factor $\{5,10\}$.
\end{itemize}

{\bf Batch Size $128$}: 
Note: (i) parameters chosen by running for $120$ epochs and picking the grid search parameter that yields the smallest validation $0/1$ error. (ii) The validation set decay scheme that we use is that if the validation error does not decay by at least $0.2\%$ every four passes over the data, we cut the learning rate by a constant factor (which is grid searched as described below). The minimal learning rate to use is fixed to be $1\times 10^{-3}$, so that we do not decay far too many times and curtail progress prematurely.
\begin{itemize}
\item \sgd: learning rate: $\{0.01,0.03,0.09,0.27,0.81\}$, learning rate decay factor $\{2,\sqrt{10},5\}$.
\item \nag/\hb: learning rate: $\{0.01,0.03,0.09,0.27\}$, momentum $\{0.5,0.8,0.9,0.95,0.97\}$, learning rate decay factor $\{2,\sqrt{10},5\}$.
\item \asgd: learning rate $\{0.01,0.03,0.09,0.27\}$, long step $\{100,1000,10000\}$, advantage parameter $\{5,10,20\}$, learning rate decay factor $\{2,\sqrt{10},5\}$.
\end{itemize}

As a final remark, for any comparison across algorithms, such as, (i) \asgd vs. \nag, (ii) \asgd vs \hb, we fix the starting learning rate, learning rate decay factor and decay schedule chosen by the best grid search run of \nag/\hb respectively and perform a grid search over the long step and advantage parameter of \asgd. In a similar manner, when we compare (iii) \sgd vs \nag or, (iv) \sgd vs. \hb, we choose the learning rate, learning rate decay factor and decay schedule of \sgd and simply sweep over the momentum parameter of \nag or \hb and choose the momentum that offers the best validation error.

We now present plots of training function value for different algorithms and batch sizes.

\textbf{Effect of minibatch sizes}:
Figure~\ref{fig:cifar-train-bs} plots training function value for batch sizes of $128$ and $8$ for \sgd, \hb and \nag. We notice that in the initial stages of training,~\nag~obtains substantial improvements compared to~\sgd~and~\hb for batch size $128$ but not for batch size $8$. Towards the end of training however,~\nag starts decreasing the training function value rapidly for both the batch sizes. The reason for this phenomenon is not clear. Note however, that at this point, the test error has already stabilized and the algorithms are just overfitting to the data.
\begin{figure}[h!]
	\centering
	\subfigure{\includegraphics[width=0.48\columnwidth]{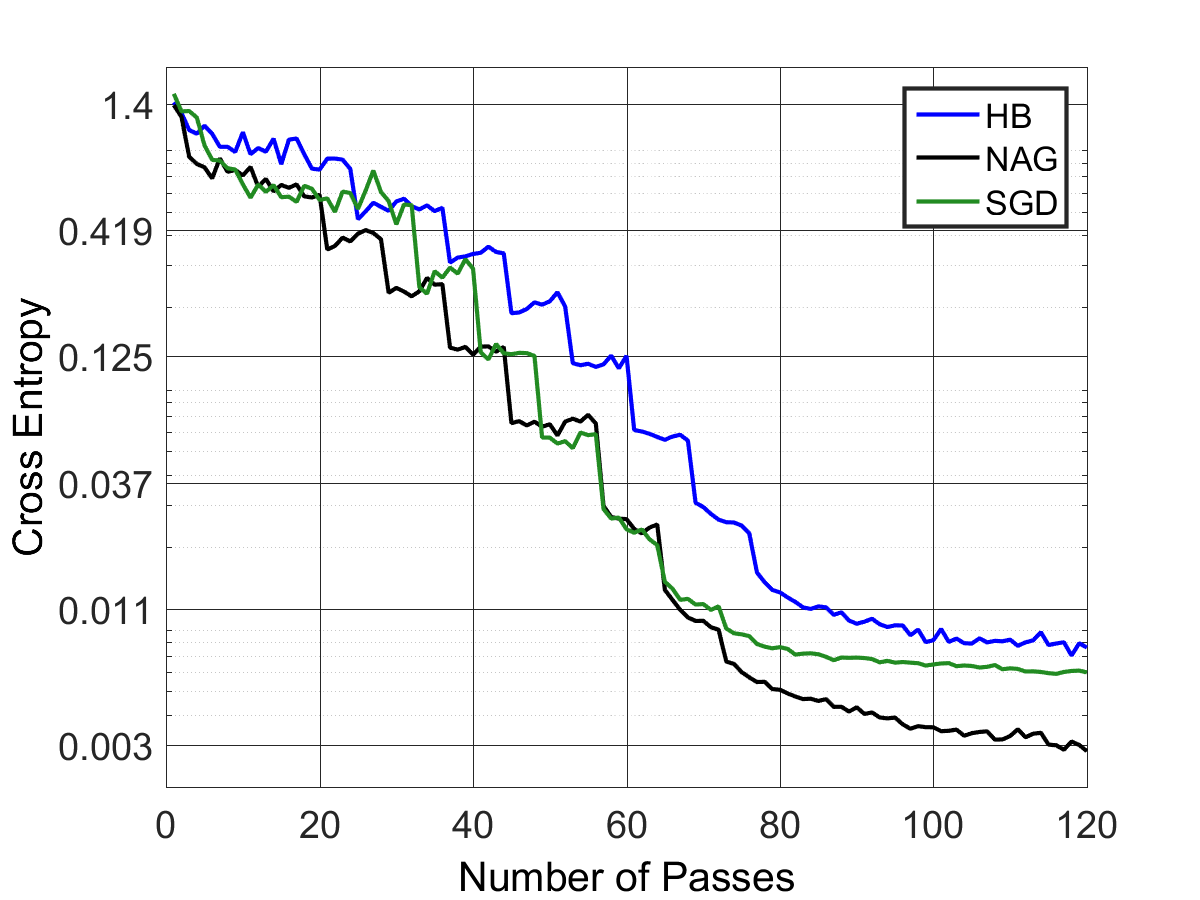}}
	\subfigure{\includegraphics[width=0.48\columnwidth]{files/figures/cifar-bs8/sgd-nag-hb-trainFunc-all-optimized-bs-8.png}}
	\caption{Training loss for batch sizes 128 and 8 respectively for~\sgd,~\hb~and~\nag.}
	\label{fig:cifar-train-bs}
\end{figure}

\textbf{Comparison of~\asgd~with momentum methods}: We now present the training error plots for~\asgd~compared to~\hb~and~\nag~in Figures~\ref{fig:cifar-train-hb-asgd} and~\ref{fig:cifar-train-nag-asgd} respectively. As mentioned earlier, in order to see a clear trend, we constrain the learning rate and decay schedule of~\asgd~to be the same as that of~\hb~and~\nag~respectively, which themselves were learned using grid search. We see similar trends as in the validation error plots from Figures~\ref{fig:cifar-test-hb-asgd} and~\ref{fig:cifar-test-nag-asgd}. Please see the figures and their captions for more details.
\begin{figure}[h!]
	\centering
	\subfigure{\includegraphics[width=0.48\columnwidth]{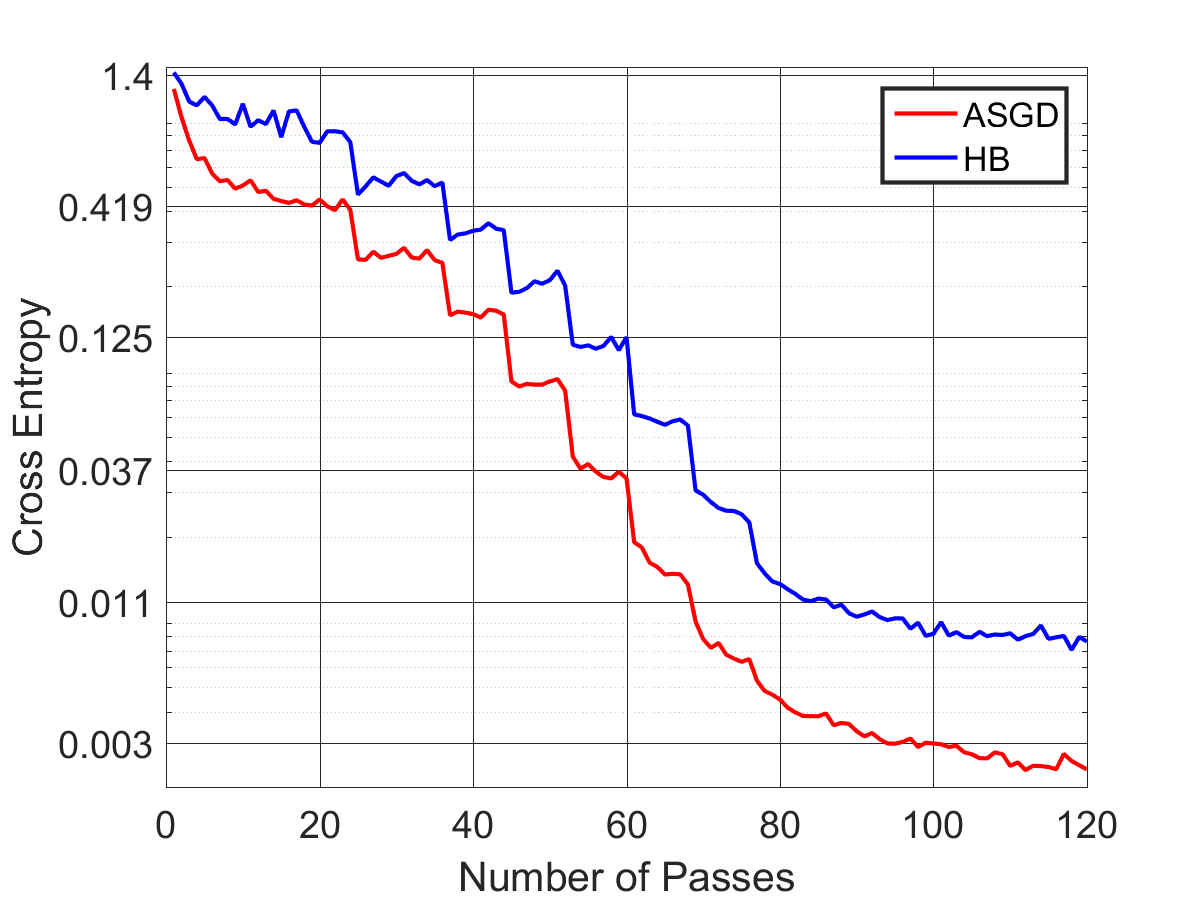}}
	\subfigure{\includegraphics[width=0.48\columnwidth]{files/figures/cifar-bs8/asgd-asgdOpt-HB-trainFunc-bs-8.png}}
	\caption{Training function value for~\asgd~compared to~\hb~for batch sizes 128 and 8 respectively.}
	\label{fig:cifar-train-hb-asgd}
\end{figure}
\begin{figure}[H]
	\centering
	\subfigure{\includegraphics[width=0.48\columnwidth]{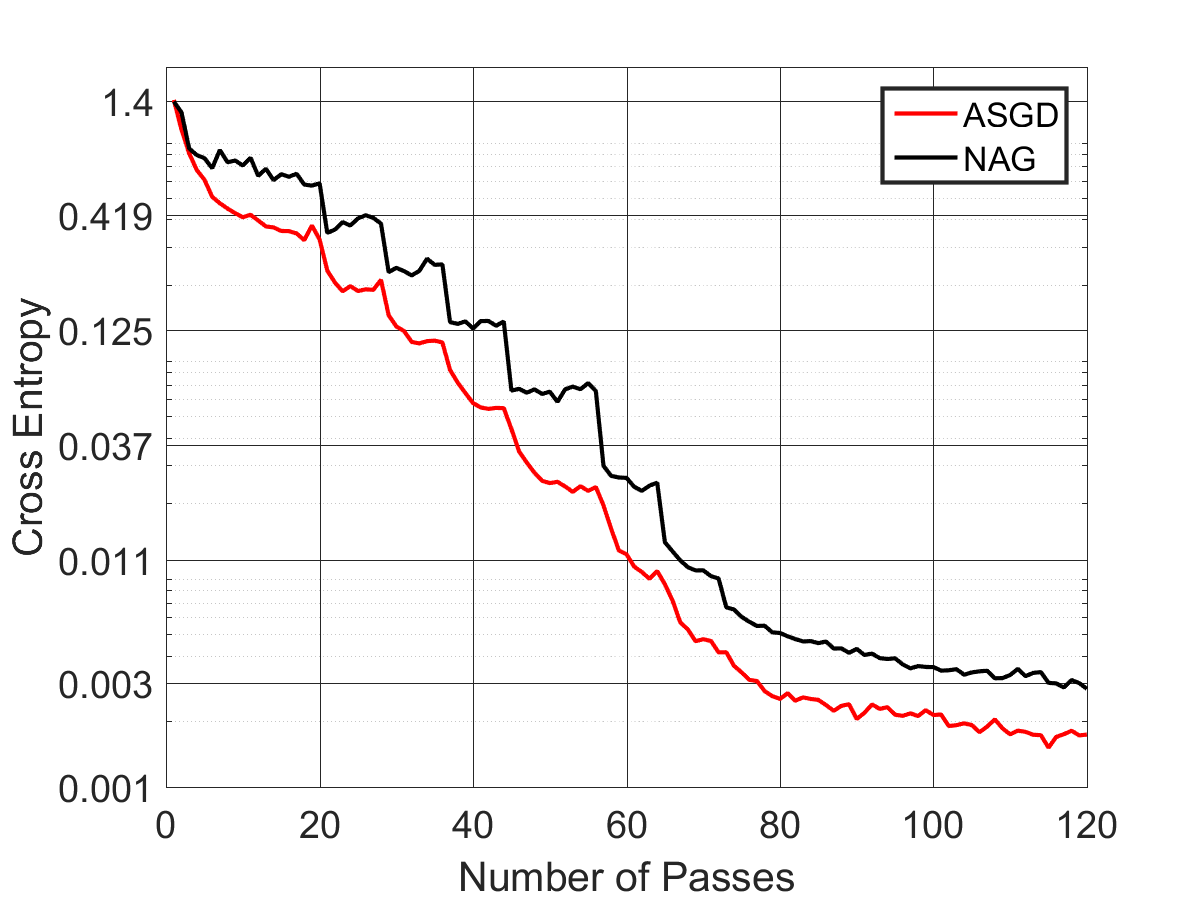}}
	\subfigure{\includegraphics[width=0.48\columnwidth]{files/figures/cifar-bs8/asgd-nag-trainFunc-bs-8.png}}
	\caption{Training function value for \asgd compared to \nag for batch size 128 and 8 respectively.}
	\label{fig:cifar-train-nag-asgd}
\end{figure}


\end{document}